\documentclass{article} 
\usepackage{iclr2027_conference,times}

\usepackage{amsmath,amsfonts,bm}

\def\eqref#1{equation~\ref{#1}}

\def\1{\bm{1}}

\DeclareMathAlphabet{\mathsfit}{\encodingdefault}{\sfdefault}{m}{sl}
\SetMathAlphabet{\mathsfit}{bold}{\encodingdefault}{\sfdefault}{bx}{n}

\usepackage{comment}
\usepackage[normalem]{ulem}
\usepackage{hyperref}
\usepackage{url}
\usepackage{amsthm}
\usepackage{booktabs}
\usepackage{multirow}
\usepackage{graphicx}
\newtheorem{definition}{Definition}
\newtheorem{theorem}{Theorem}

 \usepackage{enumitem}
 \usepackage{wrapfig}
 \usepackage{tikz}
\usetikzlibrary{positioning,arrows.meta}
\usepackage{adjustbox}
 \usetikzlibrary{shadows}
\usepackage{algorithm}
\usepackage{algpseudocode}
\title{WorldGraph: Graph-Native World Modeling}
\iclrfinalcopy
\author{
	\textbf{Zezhong Ding$^{1, 3, \dag}$, \quad Yipeng Li$^{2, 3, \dag}$, \quad Xike Xie$^{1,2, 3}$\thanks{Corresponding Author \quad  $^\dag$Equal Contribution}}\vspace{3pt} \\
     {\small $^1$School of Artificial Intelligence and Data Science, University of Science and Technology of China (USTC)} \\
	 {\small $^2$School of Biomedical Engineering, USTC} \\
	{\small $^3$Data Darkness Lab, Suzhou Institute for Advanced Research, USTC} \\ 
	\texttt{\small \{zezhongding,liyipeng0131\}@mail.ustc.edu.cn, xkxie@ustc.edu.cn} 
}

\begin{document}

\maketitle
\pagestyle{plain}
\thispagestyle{plain}
\begin{abstract}
World models infer latent states of an environment to capture its underlying dynamics and predict future evolution. 
Many real-world environments, however, are inherently relational and observed as evolving graphs, where entities, relations, and their properties change over time.
Prior graph-related world models use graph structures to organize internal states or support task-specific reasoning, rather than treating an evolving graph itself as the modeled world.
We instead study \emph{graph world modeling} (GWM), where graph evolution itself constitutes the world dynamics. 
We formulate graph world modeling over observed graph evolution, latent graph states, and heterogeneous graph-transition predictions.
Based on this formulation, we construct \textit{GWM-Zero}, a benchmark covering node-, edge-, and graph-level transitions over eight temporal graph datasets.
We propose \textit{WorldGraph}, which combines a state-aware graph transformer for multi-granularity structural and transition-conditioned evolution modeling with transition-aware GRPO using dynamic grouping and structure-aware verifiable rewards.
Extensive experiments on GWM-Zero show that WorldGraph consistently outperforms representative graph representation, temporal graph learning, graph pretraining, and graph world-model baselines across all three transition granularities.
\end{abstract}

\section{Introduction}

World models~\citep{worldmodels,DingZSZZFYSLSXL26} learn compact internal states that capture how an environment evolves to support prediction of its future. 
Many real-world environments, however, are inherently \emph{relational}: entities interact through structured relations, and both the entities and relations may change over time. 
Graphs provide a natural representation of such relational environments.
More importantly, when the environment itself is observed as an evolving graph~\citep{KurenkovLAJ0Z0023,SavinovDK18}, its evolution is the world dynamics to be modeled.
This motivates \emph{graph world models} (GWMs) that model graph states and their evolution.

At time step $t$, the world provides an observed graph $g_t$,
while its transition to $g_{t+1}$ may involve updates to nodes, edges, properties, or their combinations. We denote the {graph change} associated with this transition by $\Delta g_t$, and view an evolving graph at a high level as $\{(g_t,\Delta g_t)\}_{t}$. Beyond the observed graphs, a graph world model should infer a latent world state that captures the structural and dynamic information underlying such evolution. 
This raises a fundamental question:
\begin{center}
    {\it How should a world model represent the state underlying graph evolution and predict its heterogeneous transitions?}
\end{center}

Earlier studies have introduced graphs into world modeling mainly as auxiliary structures.
L$^3$P~\citep{Zhang0S21a} abstracts Markov decision process states and their reachability into a graph for long-horizon planning. 
C-SWM~\citep{KipfPW20} represents objects extracted from visual observations as a latent graph and uses a GNN to model their interactions.
Feng et al.~\citep{FengWLY25} use graphs to organize multi-modal structure and reasoning.
More recently, Graph-JEPA~\citep{skenderi2023graphjepa} improves the world model’s structured representation via masked subgraph latent prediction, and Liu et al.~\citep{liu2026graphworldmodelsconcepts} conceptualized these graph-enhanced world models based on relational inductive biases. 
{\it Our focus is different: we consider {\it graph-native worlds}, where the observed environment is itself an evolving graph and graph evolution is the prediction target.}
This calls for a graph-native formulation that explicitly models the latent world state underlying the graph evolution and predicts heterogeneous graph transitions across tasks of node, edge, and graph levels. 
Unlike dynamic graph learning, which typically learns temporal representations for predefined downstream tasks, a GWM maintains a latent world state that summarizes graph evolution and uses it to model heterogeneous future transitions.
Node-, edge-, and graph-level predictions can then be viewed as different observable projections of the same underlying graph-world dynamics.
This raises two fundamental challenges for a general GWM: representing the latent state underlying graph evolution and modeling the heterogeneous transitions that drive it.

\textbf{Challenge 1: State Modeling.}
A graph world state should capture not only the current graph structure but also the evolution leading to it.
Different transitions may depend on different structural scales and historical contexts.
For example, relation formation may require broader structural context, while local property changes may depend mainly on nearby neighborhoods and recent transitions. 
Existing message-passing-based GNNs~{\citep{kipf2017semi,velickovic2018graph,hamilton2017inductive}} mainly capture local structure, while graph transformers and GFMs~{\citep{sgformer,wu2022nodeformer,RampasekGDLWB22,wang2025mdgfm}} provide broader structural context but focus on static graph representations. Dynamic graph learning methods~{\citep{rossi2020temporal,peng2025tidformer}} capture temporal information, but do not explicitly model a world state from graph structure and preceding transitions. 
The first challenge is therefore to learn a world state that jointly captures multi-granularity graph structure and historical dynamics.

\begin{wrapfigure}{r}{0.38\textwidth}
\centering
\vspace{-15pt}
\begin{tikzpicture}[
    font=\scriptsize,
    >=Stealth,
    node distance=1.05cm,
    box/.style={
        draw,
        rounded corners=2.5mm,
        thick,
        minimum height=11mm,
        align=center,
        inner sep=2.5mm,
        drop shadow={shadow xshift=0.4mm, shadow yshift=-0.4mm, opacity=0.12}
    },
    boxblue/.style={
        box,
        draw=blue!50!black,
        text=blue!70!black,
        fill=blue!5,
        top color=blue!8,
        bottom color=blue!2
    },
    boxgreen/.style={
        box,
        draw=green!50!black,
        text=green!50!black,
        fill=green!5,
        top color=green!8,
        bottom color=green!2
    },
    titlebar/.style={
        rounded corners=3mm,
        minimum width=20mm,
        minimum height=7mm,
        font=\bfseries,
        align=center
    },
    arr/.style={-Latex, thick, draw=black!55},
    arrside/.style={Circle-Latex, thick, draw=black!45}
]

\node[fill=blue!3, rounded corners=5mm, minimum width=15mm, minimum height=70mm,
      opacity=0.6] at (0, 2.7) {};

\node[fill=green!3, rounded corners=5mm, minimum width=15mm, minimum height=70mm,
      opacity=0.6] at (2.7, 2.7) {};

\node[titlebar, text=blue!60!black] at (0, 7) {\scriptsize World Model};

\node[boxblue, text width=15mm] (wm-obs) at (0, 6) {Observation\\(e.g., image)};
\node[boxblue, text width=15mm, draw=blue!70!black, line width=1.1pt] (wm-state) at (0, 4.5) {State Model};
\node[boxblue, text width=15mm] (wm-latent) at (0, 3.0) {Latent State};
\node[boxblue, text width=15mm, draw=blue!70!black, line width=1.1pt] (wm-trans) at (0, 1.5) {Transition Model};
\node[boxblue, text width=15mm] (wm-next) at (0, 0.0) {Next \\Transition};

\draw[arr] (wm-obs)    -- (wm-state);
\draw[arr] (wm-state)  -- (wm-latent);
\draw[arr] (wm-latent)  -- (wm-trans);
\draw[arrside] (wm-obs.east) -- ++(0.3,0) |- (wm-trans.east);
\draw[arr] (wm-trans)   -- (wm-next);

\draw[draw=black!18, dashed, thick] (1.5, 6.2) -- (1.5, -0.6);

\begin{scope}[xshift=2.7cm]
    \node[titlebar, text=green!50!black] at (0, 7) {\scriptsize GWM};

    \node[boxgreen, text width=15mm] (gw-observed) at (0, 6) {Graph Observation\\$g_t, \Delta g_{t-1}$};
    \node[boxgreen, text width=15mm, draw=green!70!black, line width=1.1pt] (gw-state) at (0, 4.5) {State Model $f_s$};
    \node[boxgreen, text width=15mm] (gw-latent) at (0, 3.0) {Latent State $\mathbf{s}_t$};
    \node[boxgreen, text width=15mm, draw=green!70!black, line width=1.1pt] (gw-trans) at (0, 1.5) {Transition Model $f_\Delta$};
    \node[boxgreen, text width=15mm] (gw-next) at (0, 0.0) {Next Graph\\Change $\Delta g_t$};

    \draw[arr] (gw-observed)    -- (gw-state);
     \draw[arr] (gw-state)      -- (gw-latent);
    \draw[arr] (gw-latent)      -- (gw-trans);
    \draw[arrside] (gw-observed.east) -- ++(0.3,0) |- (gw-trans.east);
    \draw[arr] (gw-trans)       -- (gw-next);
\end{scope}

\end{tikzpicture}
\label{fig:intro}
\vspace{-15pt}
\caption{\small \bf World Model \citep{worldmodels} vs. Our Proposed GWM}
\vspace{-18pt}
\label{fig:world_model_comparison}
\end{wrapfigure}
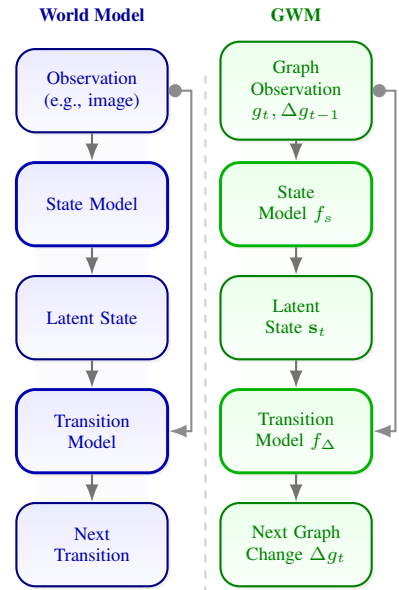

\textbf{Challenge 2: Transition Modeling.}
Graph transitions are inherently heterogeneous: 
they may modify node properties, relations, or local structures, while different transition patterns can occur at highly imbalanced frequencies~\citep{liu2023mata}. 
Moreover, their significance is not determined by frequency alone: a rare change involving structurally important entities may have a substantial effect on subsequent graph evolution.
A general GWM should thus avoid being dominated by frequent transitions while remaining sensitive to structurally consequential changes.
The second challenge is to learn heterogeneous graph transitions while accounting for both transition frequency and structural relevance.

To address these challenges, we propose \textit{WorldGraph}, which jointly models latent world states and heterogeneous transitions. 
For {\bf Challenge 1}, we develop a \textit{state-aware graph transformer} that constructs the world state by integrating multi-granularity structural encoding and history-aware state encoding.
It combines hop- and path-level graph contexts to capture structural dependencies at different scales, while incorporating preceding graph transitions to model how the current graph state has evolved over time. 
For {\bf Challenge 2}, we develop a \textit{transition-aware GRPO} with \emph{dynamic grouping} and
\emph{structure-aware verifiable rewards}. 
Dynamic grouping allocates more training signal to rare transitions, while structure-aware rewards emphasize changes involving structurally important and historically variable entities. 
Together, these techniques enable WorldGraph to model graph dynamics across node-, edge-, and graph-level transitions within a unified framework.

Our contributions are threefold.
    \textbf{1)} We formulate graph-native
    world modeling through observed graph states, transitions, and latent world states, and construct {\it GWM-Zero}, a benchmark covering node-, edge-, and graph-level transition prediction tasks over \textbf{8} graph datasets.
    \textbf{2)} We propose \textit{WorldGraph}, a unified framework that instantiates this formulation with a
    state-aware graph transformer and transition-aware RL.
    \textbf{3)} {Extensive experiments demonstrate that WorldGraph achieves an average model quality improvement of \textbf{10.77\%} over the strongest baselines, with the largest improvement of \textbf{50.66\%} in F1 score on the TGBN-Trade~\citep{huang2023temporal} dataset.}

\vspace{-10pt}
\section{Problem Formulation}
\label{sec:problem}

We formalize graph world modeling through observed graph states, transitions, and world states.
\paragraph{Observed Graph Evolution.}
At time step $t$, the observed {\it graph state} is represented as
$g_t=(V_t,E_t,X_t)\in\mathcal{G}$,
where $V_t$, $E_t$, and $X_t$ denote the entities, relations, and associated properties, respectively.
The evolution from $g_t$ to $g_{t+1}$ forms a {\it graph transition}.
We use $\Delta g_t$ to denote the {\it graph change} associated with this transition, which may involve 
node or edge addition/deletion, property changes, or their combinations.
Accordingly, an evolving graph world can be viewed as a sequence of graph states and their changes, {$\{(g_t,\Delta g_t)\}_{t=1}^{T}$}.

\paragraph{Latent World States.}
While $g_t$ describes the graph observed at time $t$, the underlying dynamics of graph evolution are not directly observed.
We therefore introduce a world state $\mathbf{s}_t\in\mathcal{S}$ 
that summarizes both the current graph structure and the preceding evolution:
\begin{equation}
\label{eq:state_model}
\mathbf{s}_t=f_s(g_t,\Delta g_{t-1},\mathbf{s}_{t-1}),
\end{equation}
where $f_s$ denotes the {\it state model}.

\begin{definition}[\textbf{Graph World Model}]
\label{def:gwmreal}
A graph world model (GWM) models graph-world dynamics through a state model $f_s$ and a transition model $f_{\Delta}$.
Given the current graph observation $g_t$, the preceding graph change $\Delta g_{t-1}$, and the previous world state $\mathbf{s}_{t-1}$, the state model infers $\mathbf{s}_t=f_s(g_t,\Delta g_{t-1},\mathbf{s}_{t-1})$, while the transition model 
predicts the next graph change as
$\Delta g_t = f_\Delta(g_t,\Delta g_{t-1},\mathbf{s}_{t})$.
\end{definition}

The two components capture complementary aspects of graph-world dynamics.
The state model captures the latent structural and dynamic information of the evolving graph world, while the transition model predicts its subsequent evolution.

A complete graph change $\Delta g_t$ may contain multiple structural events and may be observed at different granularities.
We therefore instantiate graph-transition prediction at three levels: node-level changes in entities and properties, edge-level changes in relations, and graph-level changes in local structures.
Importantly, these tasks are different observable projections of the same graph-world transition $\Delta g_t$ rather than separate forms of the world dynamics.
Accordingly, graph world modeling consists of two coupled problems: \emph{state modeling} for learning the latent world state $\mathbf{s}_t$, and \emph{transition modeling}, which predicts observable graph transitions from this state.

\vspace{-10pt}
\section{WorldGraph}
\vspace{-10pt}
\label{sec:worldgraph}

WorldGraph instantiates the two components of a GWM introduced in Section~\ref{sec:problem}.
For state modeling, it constructs $s_t$ by combining multi-granularity structure in the current graph with transition-conditioned history.
For transition modeling, it learns heterogeneous graph transitions with frequency-aware sampling and structure-aware rewards.
Figure~\ref{fig:framework} gives an overview of the framework.

\begin{figure*}[t]
    \centering
\includegraphics[width=0.9\textwidth]{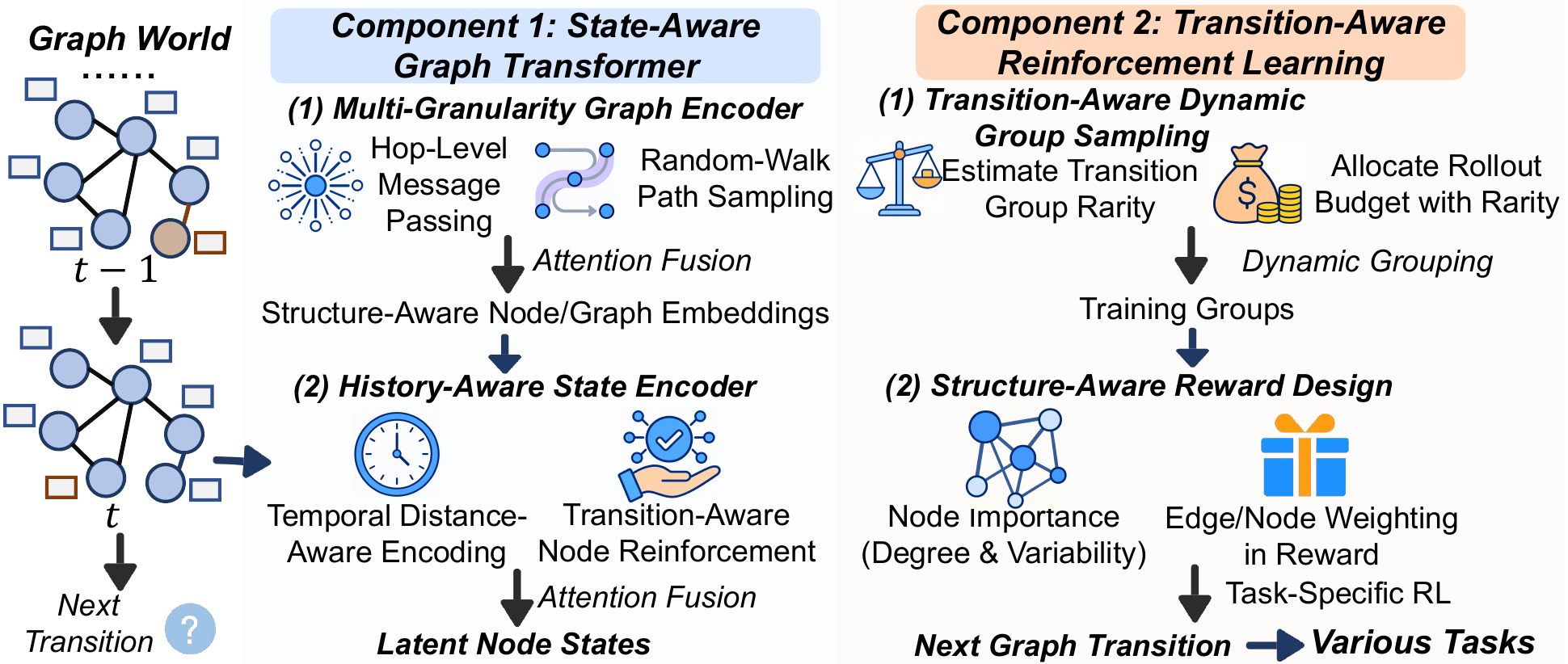}
    \vspace{-12pt}
    \caption{\small \textbf{Overview of WorldGraph}. {
    The \textit{state-aware graph transformer} constructs latent world states by integrating multi-granularity graph structure with transition-conditioned history. 
The \textit{transition-aware RL algorithm} learns heterogeneous graph transitions through dynamic grouping and structure-aware verifiable rewards across node-, edge-, and graph-level tasks.
    }}
    \label{fig:framework}
    \vspace{-10pt}
\end{figure*}
\vspace{-5pt}
\subsection{State-aware Graph Transformer}
\vspace{-5pt}
The latent world state $s_t$ should summarize both what the current graph looks like and how it arrived there.
WorldGraph constructs this state in two stages: a multi-granularity graph encoder captures the current structure, and a history-aware state encoder integrates the preceding evolution.
\subsubsection{Stage 1: Multi-granularity Graph Encoder}
\vspace{-5pt}
Different graph transitions depend on structural information at different scales\footnote{For example, in the UN Vote dataset~\citep{poursafaei2022towards}, two countries without a direct co-voting edge may be linked by a multi-hop chain of shared co-voting partners, so predicting a new edge requires long-range context. In the TGBN-Genre dataset~\citep{huang2023temporal}, predicting whether a user’s music preferences change mainly depends on that user’s recent preference history, so this local information is sufficient.}.
We combine hop-based message passing and random-walk path sampling to obtain complementary structural views.

\textbf{Hop-Level Message Passing.}
Given a maximum receptive field constraint (i.e., the maximum number of hops $L_{\max}$) {and the graph state $g_t =(V_t,E_t,X_t)$}, 
for every node $v \in V_t$,
{we obtain the hop-level node embeddings $\{\mathbf{x}_v^{(\ell)}\}_{1 \leq \ell \leq L_{\max}}$, where $\mathbf{x}_v^{(\ell)}$ aggregates the node features from all neighbors within the $\ell$-hop range of node $v$ by iteratively applying message passing~\citep{hamilton2017inductive}.}\\
\textbf{Random-Walk Path Sampling.}
To further enhance the richness of structural representations, we also perform random-walk sampling to collect $M$ paths starting from every node $v \in V_t$.
{Then, for each node $v\in V_t$, we obtain the path-level embeddings $\{\mathbf{x}_v^{\mathrm{path}_i}\}_{i=1}^{M}$, where $\mathbf{x}_v^{\mathrm{path}_i}$  aggregates all node features along $\mathrm{path}_i$ by applying an attention mechanism~\citep{vaswani2017attention}.}\\
\textbf{Multi-Granularity Attention Fusion.}
 {For each node $v\in V_t$ at time $t$, after obtaining $\{\mathbf{x}_v^{(\ell)}\}_{\ell=1}^{L_{\max}}$ and $\{\mathbf{x}_v^{\mathrm{path}_i}\}_{i=1}^{M}$, we fuse these embeddings via an attention-based aggregation\footnote{Here, we fuse the hop-level embeddings and the path-level embeddings, while also incorporating global context from a linear-attention-aware subgraph~\citep{sgformer}.} to obtain a structure-aware node embedding $\mathbf{z}_{v,t}$. 
 For $g_t$, we further obtain the structure-aware graph embedding $\mathbf{z}_t$ by applying mean pooling over the node embeddings $\{\mathbf{z}_{v,t}\}_{v \in V_t}$}.

\subsubsection{Stage 2: History-aware State Encoder}
\vspace{-5pt}
The current graph structure alone does not reveal how the graph arrived at its present state.
We therefore integrate preceding graph changes into the latent world state, while accounting for both their temporal distance and their relevance to recent evolution.

\textbf{Temporal Distance-Aware Encoding.}
Because the states that are closer in time have a greater impact on the current transition~\citep{wang2025understanding}, we consider incorporating the temporal distance (i.e., the difference between the current time and historical time) into the latent states at first.

Given the current time $t$ and a historical time $\tau < t$,
we obtain the transition representation of $\Delta g_{\tau}$ as follows:
$\mathbf{x}_{ \Delta g_{\tau}}
=\mathrm{Type}(\Delta g_{\tau})$,
where 
$\mathrm{Type}(\cdot)$ encodes the graph transition type\footnote{We use a codebook that maps each transition type to its corresponding learnable embedding.}.
Then, for every node $v\in V_t$, we calculate the node memory embeddings $\{\mathbf{m}_{v,\tau}\}$ for $\tau < t$\footref{ft:deltat}, where
\begin{equation}
\mathbf{m}_{v,\tau}=
\mathbf{W}_x\mathbf{z}_{v,\tau}
+\mathbf{W}_z\mathbf{z}_{\tau}
+\mathrm{Dist}(t,\tau)
+\mathbf{W}_a\mathbf{x}_{\Delta g_\tau},
\end{equation}
where $\mathbf{W}_x$, $\mathbf{W}_z$, and $\mathbf{W}_a$ are learnable weight matrices used for projection. 
$\mathrm{Dist}(t,\tau)$ encodes the relative temporal distance $t-\tau$\footnote{\label{ft:deltat}In practice, we compute the node memory embeddings for the history period up to the past $t_\mathrm{max}$ (i.e., $t - t_\mathrm{max} \leq \tau < t$), which is discussed in Appendix~\ref{sec:ap_rl_sensitivity}.}. Then, the node memory embedding is used to obtain the latent state. $\mathbf{W}_a\mathbf{x}_{\Delta g_\tau}$
allows the node memory embedding to incorporate the graph transition type.

\textbf{History-Aware Attention Fusion.}
For each node $v$ at time $t$, we construct a query embedding $\mathbf{q}_{v,t}$ by concatenating $\mathbf{z}_{v,t}$, $\mathbf{z}_{t}$, $\mathbf{x}_{\Delta g_{t-1}}$, and its {previous latent state $\mathbf{s}_{v,t-1}$}.
Then, we aggregate its historical memory $\{\mathbf{m}_{v,\tau}\}$ with multi-head attention.
For one attention head, the attention score from history $\tau$ to the current time $t$ is: $\beta_{v,t,\tau} =
\frac{(\mathbf{q}_{v,t} \mathbf{W}_Q)\,(\mathbf{m}_{v,\tau} \mathbf{W}_K)^\top}{\sqrt{d}}
\;+\; \mathbf{p}_{t-\tau}
\;+\; \Pi(\Delta g_{\tau},\Delta g_{t-1})$, 
where $\mathbf{W}_Q$ and $\mathbf{W}_K$ are learnable linear projection matrices,
and $d$ is the dimension of the projected node embedding.
{$\Pi(\Delta g_{\tau},\Delta g_{t-1})
= \left(\mathbf{x}_{\Delta g_{\tau}}\right)^\top \mathbf{x}_{\Delta g_{t-1}}$ is the similarity between the two transitions. This similarity is used as the reinforcement to account for how previous transitions ($\Delta g_{\tau}$) influence recent dynamics ($\Delta g_{t-1}$, which is the most recently observed transition).}\\
Then, we obtain $\tilde{\beta}_{v,t,\tau}$ using exponential normalization across $\tau$. Based on this,
the history embedding for node $v$ (one attention head) is $\mathbf{h}_{v,t}
=\sum_\tau\tilde{\beta}_{v,t,\tau}\,
\left(\mathbf{m}_{v,\tau}\mathbf{W}_V\right)$,
where $\mathbf{W}_V$ is the learnable linear projection matrix. The final history embedding is obtained by mixing multiple attention heads with an MLP, yielding $\tilde{\mathbf{h}}_{v,t}$.\\
Finally, we sum \( \mathbf{z}_{v,t} \), the history embedding \( \tilde{\mathbf{h}}_{v,t} \), and the previous latent state \( \mathbf{s}_{v,t-1} \) to obtain the latent node state at the current time \( \mathbf{s}_{v,t} \) (LayerNorm~\citep{ba2016layer} is used to stabilize training):
\begin{equation}
\label{eq:finalstate}
\mathbf{s}_{v,t} = \mathrm{LayerNorm}(\tilde{\mathbf{h}}_{v,t} +  \mathbf{z}_{v,t} +  \mathbf{s}_{v,t-1} ).    
\end{equation}

\vspace{-5pt}
\subsubsection{Theoretical Analysis of State-Aware Graph Transformer}
\vspace{-5pt}
Here, we provide a theoretical analysis of why our state-aware graph transformer has advantages over the current state-of-the-art graph transformer, i.e., SGFormer~\citep{sgformer}, in terms of state modeling.
For this analysis, we start from the message-passing mechanisms of SGFormer and WorldGraph, and then derive an upper bound on the difference between the final structure-aware node embedding they produce and the ideal structure-aware node embedding (i.e., the embedding obtained when the sampled subgraph provides full coverage).

Since SGFormer is trained using a subgraph with a fixed hop size, let the obtained subgraph be \(g^{(\mathrm{SG})}\). In contrast, our graph transformer uses multiple-hop subgraphs and random-walk sampled subgraphs. These subgraphs are denoted as \(\{g^{(\mathrm{OURS})}_i\}_{i=1}^{L_{\max}+M}\). Ideally, the subgraphs are generated from neighborhoods spanning arbitrarily large hop distances,
and the random-walk sampling is performed with a sufficiently large number of steps.
Let $\zeta$ denote the (sufficient) number of  subgraphs.
We obtain \(\{g^{(\mathrm{IDEAL})}_i\}_{i=1}^{\zeta}\).
By substituting these into $\mathbf{s}_{v,t}$, we can obtain the latent states produced by different mechanisms, which we denote as \(\mathbf{s}_{v,t}^{(\mathrm{SG})}\), \(\mathbf{s}_{v,t}^{(\mathrm{OURS})}\), and \(\mathbf{s}_{v,t}^{(\mathrm{IDEAL})}\).
\begin{theorem}
\label{thm:graphtransformer}
    The difference between \(\mathbf{s}_{v,t}^{(\mathrm{OURS})}\) and \(\mathbf{s}_{v,t}^{(\mathrm{IDEAL})}\) is \(\Delta_{\mathrm{OURS}}=\left\|\mathbf{s}_{v,t}^{(\mathrm{OURS})}-\mathbf{s}_{v,t}^{(\mathrm{IDEAL})}\right\|_{2}\), and 
    the difference between \(\mathbf{s}_{v,t}^{(\mathrm{SG})}\) and \(\mathbf{s}_{v,t}^{(\mathrm{IDEAL})}\) is \(\Delta_{\mathrm{SG}}=\left\|\mathbf{s}_{v,t}^{(\mathrm{SG})}-\mathbf{s}_{v,t}^{(\mathrm{IDEAL})}\right\|_{2}\). The upper bound of \(\Delta_{\mathrm{OURS}}\) is smaller than that of \(\Delta_{\mathrm{SG}}\).
\end{theorem}
\vspace{-10pt}
\begin{proof}
    The detailed proof is provided in Appendix~\ref{sec:ap-c-thm1}.
\end{proof}

\vspace{-5pt}
\subsection{Transition-Aware Reinforcement Learning}
\vspace{-5pt}
Given the latent world state $s_t$, the transition model predicts heterogeneous graph changes with highly imbalanced frequencies and varying structural relevance.
To account for these differences during optimization, we formulate transition learning with verifiable rewards and adapt GRPO to graph-world transitions.
Transition-aware dynamic grouping emphasizes rare transitions, while structure-aware rewards account for structural relevance.

\subsubsection{Transition-Aware Dynamic Group Sampling}
Here, we introduce transition-aware dynamic grouping: the group assignment at time step $t$ is biased towards rare historical transitions (i.e., the transitions that appear less frequently in the history). This design increases the coverage of low-frequency transitions, providing them with more training signal. As a result, the algorithm can learn robust policies for rare transitions.

\textbf{Dynamic Transition Frequency Estimation.}
To achieve such a dynamic group sampling algorithm, we need to estimate the transition frequency at first.
At time $t$, let the transition-conditioned groups be $\mathcal{J}_t$. {Each group $J_i\in\mathcal{J}_t$ corresponds to a specific graph transition type and stores the transitions from time steps before $t$ that belong to this transition type.}
For a transition group $J_i\in\mathcal{J}_t$, we define the size of group $J_i$ as  $|J_i|$.
Based on $|J_i|$, we can compute the frequency of transitions of different types before time step 
$t$ as $\widehat{p}(J_i)$. \\
{We then assign a rarity weight to each transition group, so that groups appearing less frequently in the prefix receive larger weights: $\omega(J_i)=\left(\widehat{p}(J_i)\right)^{-\gamma}, \gamma>0$\footnote{\label{ft:para}In practice, we set $\gamma$ as 1, $\theta$ as 1, and $\eta$ as 1, as discussed in Appendix~\ref{sec:ap_rl_sensitivity}.},
and the probability of selecting group $J_i$ is $\mathrm{Prob}(J_i)
    =
    \frac{\exp\!\left(\omega(J_i)\right)}
    {\sum_{J_j\in\mathcal{J}_t}
    \exp\!\left(\omega(J_j)\right)}$.\\
\textbf{Rarity-based Rollout Budget Allocation.} Finally, to make earlier rare transitions receive a larger effective group size, we allocate the total rollout budget $\mathcal{B}$ across groups proportionally. 
We reserve the same minimum number of rollouts\footnote{\label{ft:minmumrollout}In practice, we set the minimum number of rollouts as 2. We do this to ensure every  group receives at least a minimal number of rollouts, preventing overly noisy return estimates from groups with too few samples.} for every group and use largest-remainder rounding for the remaining budget, such that $\sum_{J_i}B_t(J_i)=\mathcal{B}$, where $B_t(J_i)$ is the allocated number of rollouts for transition group $J_i$ at time $t$.
Thus, groups corresponding to transitions that are rarer in the history are assigned more rollout samples at time $t$, {increasing the relative training signal for these frequency-imbalanced transition patterns.}

\vspace{-5pt}
\subsubsection{Structure-aware Reward Design}
\vspace{-5pt}
Graph transitions differ in their structural relevance.
We therefore weight the verifiable reward using two complementary signals: structural importance and historical variability, emphasizing transitions involving more relevant nodes.

\textbf{Node Importance Calculation.}
For a node $v\in V_t$, we define a structural importance term based on its node degree ($\deg(\cdot)$): $\mathcal{I}_{\mathbf{s}}(v)=\left(\deg(v)\right)^{\theta}$, where $\theta>0$\footref{ft:para} controls how strongly we emphasize highly-connected nodes.
Let $f_{<t}(v)$ be the empirical frequency that node $v$ is involved in a change (i.e., the fraction of time steps at which node $v$ is modified),
and define the variability term as: $\mathcal{I}_\mathbf{v}(v)=\left(f_{<t}(v)\right)^{\eta}$,
with $\eta>0$\footref{ft:para}.
We combine the two terms into a single importance: $\mathcal{I}(v)=\frac{\mathcal{I}_\mathbf{s}(v) \cdot \mathcal{I}_\mathbf{v}(v)}{\sum_{v'\in V_t}\left(\mathcal{I}_\mathbf{s}(v') \cdot\mathcal{I}_\mathbf{v}(v')\right)}$.
Based on the node importance, we design a structure-aware reward.

{\textbf{Task-Specific Structure-Aware Reward Design.}
Graph-transition tasks, including node-, edge-, and graph-level tasks, produce outputs at different granularities. We therefore 
define a verifiable reward consistent with the output and evaluation metric. 
Then, we define the RL reward as follows.\\
{For node-level tasks\footnote{For edge-level (or graph-level) tasks, we use the change edge (or graph) sets, respectively.}, we denote the predicted and ground-truth change node sets by $\Delta\hat{Q}$ and $\Delta Q$, respectively. 
A prediction is counted as a true positive only when it exactly matches a ground-truth change. 
Therefore, we define the precision $\mathcal{P}$, recall $\mathcal{R}$, and the structure-aware reward $R_{\mathrm{structure}}$:
\begin{equation}
\vspace{-3pt}
    \mathcal{P} = \frac{\sum_{v \in\Delta\hat{Q}\cap\Delta Q}\mathcal{I}(v)}{\sum_{v\in\Delta\hat{Q}}\mathcal{I}(v)}, \quad \mathcal{R}=\frac{\sum_{v \in\Delta\hat{Q}\cap\Delta Q}\mathcal{I}(v)}{ \sum_{v\in\Delta Q}\mathcal{I}(v)}, \quad 
    R_{\mathrm{structure}}= \dfrac{2\,\mathcal{P}\,\mathcal{R}}{\mathcal{P}+\mathcal{R}}.
\end{equation}
Based on this, we design rewards for node-level, edge-level, and graph-level tasks as follows.\\
{Although the three graph-world tasks operate at different granularities, their rewards share the same principle: a prediction should identify the correct structural changes and estimate the corresponding {property} changes. 
We therefore use the unified reward $R=\lambda R_{\mathrm{structure}}+(1-\lambda)R_{\mathrm{property}}$,
where  $R_{\mathrm{property}}$ is calculated based on the ground-truth property and the predicted property. $\lambda$ balances the two terms. The concrete definitions for node-, edge-, and graph-level tasks are provided in Appendix~\ref{sec:task_specific_reward_details}. Based on these rewards, we train the policy network 
$\Delta g_t = f_\Delta(g_t,\Delta g_{t-1},\mathbf{s}_{t})$\footnote{$\mathbf{s}_t$ includes all latent node states $\{\mathbf{s}_{v,t}\}$ in $g_t$. $f_\Delta$ is an MLP.}
to maximize the expected cumulative return and obtain the transition $\Delta g_t$ at time step $t$.
\vspace{-5pt}
\subsubsection{Theoretical Analysis of Transition-Aware RL Algorithm}
\vspace{-5pt}
In GRPO~\citep{shao2024deepseekmathpushinglimitsmathematical}, the group size is the same for different transitions. The distribution of group sizes across different transitions can be viewed as uniform. In our transition-aware RL algorithm, our transition group-size distribution is determined based on the historical interaction frequency of each transition.
Following~\citep{kim2026efficiency}, we use the effective variance ratio (EVR). A smaller EVR means lower gradient noise and more stable RL training. 
\begin{theorem}
\label{thm:rl}
In GRPO, the group size 
is
$\rho^{\mathrm{GRPO}}=C_a$, where $C_a$ is a constant for different graph transition types. In our RL algorithm, the group size
corresponding to each transition type (the corresponding group is $J_i$) is
$\rho^{\mathrm{OURS}}=C_a + \mathrm{Prob}(J_i)\cdot \mathcal{B}$.
Based on~\citep{kim2026efficiency}, for a group size $\rho_i$, an upper bound on EVR is
$\mathrm{EVR}(\rho_i)=1+\frac{\delta-1}{2\rho_i} + M_\mathrm{EVR}$,
where $M_\mathrm{EVR} \geq \mathcal{O}\!\left(\frac{1}{\rho_i^2}\right)$ is a constant and $\delta$ is the kurtosis of the reward distribution~\citep{kim2026efficiency}.
Then, for a given transition type, the group produced by our RL algorithm has a smaller upper bound on the EVR than the group produced by GRPO.
\end{theorem}
\vspace{-10pt}
\begin{proof}
    The detailed proof is provided in Appendix~\ref{sec:ap-c-thm2}.
\end{proof}
\vspace{-10pt}
}
}

\begin{table*}[!t]
    \centering
    \caption{\small \textbf{Overall Performance on the 3 Graph-Transition Tasks.} 
    {Add., Rem., and Sem. denote addition, deletion, and property-change F1, respectively. Higher is better for F1, Macro-F1, and NDCG@10, while lower is better for MAE and RMSE. The best results are shown in \textbf{bold} and the second best results are \underline{underlined}. We use an existing Controller design~\citep{dinella2020hoppity} to enable the baselines in the first three categories to predict node-, edge-, and graph-level changes.}}
    \label{tab:worldgraph_results}
    \scriptsize

    {\small\bfseries (a) Node-Level Tasks\par}\vspace{0.8mm}
    \setlength{\tabcolsep}{2.0pt}%
    \renewcommand{\arraystretch}{1.12}%
    \resizebox{\linewidth}{!}{%
    \begin{tabular}{l|l|cccc|cccc|cccc}
        \toprule
        \multirow{2}{*}{Method} & \multirow{2}{*}{Type} & \multicolumn{4}{c|}{TGBN-Trade} & \multicolumn{4}{c|}{TGBN-Genre} & \multicolumn{4}{c}{TGBN-Reddit} \\
        & & Add. F1 & Rem. F1 & Sem. F1 & NDCG@10 & Add. F1 & Rem. F1 & Sem. F1 & NDCG@10 & Add. F1 & Rem. F1 & Sem. F1 & NDCG@10 \\
        \midrule
        GCN & MPNN & 0.4718 & 0.1540 & 0.5741 & 0.6357 & 0.7112 & 0.8488 & 0.6375 & 0.2469 & 0.9706 & 0.8976 & 0.6284 & 0.2161 \\
        GAT & MPNN & 0.5248 & 0.1385 & 0.6275 & 0.6386 & 0.7253 & 0.8598 & 0.6523 & 0.2699 & 0.9720 & 0.8972 & 0.6440 & 0.2237 \\
        GraphSAGE & MPNN & 0.5864 & 0.0000 & 0.6595 & 0.6246 & 0.7256 & 0.8538 & 0.6496 & 0.2780 & 0.9705 & 0.8902 & 0.6435 & 0.2519 \\
        SGFormer & Graph Transformer & 0.6874 & 0.0308 & 0.6732 & 0.6324 & 0.7252 & 0.8564 & 0.6497 & 0.2822 & 0.9731 & 0.9024 & 0.6445 & 0.2413 \\
        NodeFormer & Graph Transformer & 0.7069 & 0.0878 & 0.6669 & 0.6336 & 0.7080 & 0.8394 & 0.6470 & 0.2344 & 0.8011 & 0.7329 & 0.5409 & 0.1275 \\
        GraphGPS & Graph Transformer & 0.8033 & 0.0000 & 0.6542 & 0.6155 & 0.7302 & 0.8596 & 0.6533 & 0.2824 & 0.9684 & 0.8931 & 0.6449 & 0.2376 \\
        Graph-JEPA & Pretrained & 0.5585 & 0.0419 & 0.5777 & 0.6357 & 0.6563 & 0.7634 & 0.6143 & 0.2487 & 0.9573 & 0.8861 & 0.5999 & 0.1520 \\
        MDGFM & Pretrained & 0.4497 & 0.1517 & 0.5616 & 0.6406 & 0.7194 & 0.8426 & 0.6177 & 0.2107 & 0.9558 & 0.8607 & 0.6151 & 0.1676 \\
        TGN & Temporal Model & 0.8492 & 0.0559 & 0.6666 & 0.6181 & 0.7378 & 0.8698 & 0.6505 & 0.2683 & 0.9726 & 0.8982 & 0.6457 & 0.2259 \\
        TIDFormer & Temporal Model & 0.8442 & 0.1129 & 0.6740 & 0.6299 & 0.7409 & 0.8754 & 0.6529 & 0.2556 & 0.9705 & 0.9066 & 0.6429 & 0.2150 \\
        \midrule
        \textbf{WorldGraph} & \textbf{GWM} & \underline{0.9245} & \underline{0.6430} & \underline{0.7117} & \textbf{0.6484} & \underline{0.7460} & \underline{0.8801} & \underline{0.6557} & \underline{0.4418} & \underline{0.9755} & \underline{0.9317} & \underline{0.6458} & \underline{0.3838} \\
        \textbf{WorldGraph-Pretrain} & \textbf{GWM} & \textbf{0.9696} & \textbf{0.6606} & \textbf{0.7354} & \underline{0.6469} & \textbf{0.7474} & \textbf{0.8849} & \textbf{0.6685} & \textbf{0.4448} & \textbf{0.9779} & \textbf{0.9561} & \textbf{0.6492} & \textbf{0.3901} \\
        \bottomrule
    \end{tabular}}

    \vspace{2.0mm}
    {\small\bfseries (b) Edge-Level Tasks\par}\vspace{0.8mm}
    \setlength{\tabcolsep}{2.0pt}%
    \renewcommand{\arraystretch}{1.12}%
    \resizebox{\linewidth}{!}{%
    \begin{tabular}{l|l|ccc|ccc|ccc|ccc}
        \toprule
        \multirow{2}{*}{Method} & \multirow{2}{*}{Type} & \multicolumn{3}{c|}{TGBN-Trade} & \multicolumn{3}{c|}{UN Vote} & \multicolumn{3}{c|}{Contact} & \multicolumn{3}{c}{SocialEvo} \\
        & & Add. F1 & Rem. F1 & Macro-F1 & Add. F1 & Rem. F1 & Macro-F1 & Add. F1 & Rem. F1 & Macro-F1 & Add. F1 & Rem. F1 & Macro-F1 \\
        \midrule
        GCN & MPNN & 0.2633 & 0.2973 & 0.2803 & 0.1337 & 0.1025 & 0.1181 & 0.0146 & 0.7946 & 0.4046 & 0.2387 & 0.5554 & 0.3970 \\
        GAT & MPNN & 0.2967 & 0.3841 & 0.3404 & 0.1393 & 0.0926 & 0.1159 & 0.0146 & 0.7518 & 0.3832 & 0.2337 & 0.5897 & 0.4117 \\
        GraphSAGE & MPNN & 0.3114 & 0.3865 & 0.3489 & 0.1649 & 0.1031 & 0.1340 & 0.0170 & 0.8174 & 0.4172 & 0.2629 & 0.5556 & 0.4093 \\
        SGFormer & Graph Transformer & 0.3174 & 0.3915 & 0.3544 & 0.1403 & 0.1058 & 0.1230 & 0.0141 & 0.8144 & 0.4142 & 0.2175 & 0.5233 & 0.3704 \\
        NodeFormer & Graph Transformer & 0.3031 & 0.3797 & 0.3414 & 0.1610 & 0.0840 & 0.1225 & 0.0124 & 0.5443 & 0.2784 & 0.1532 & 0.3496 & 0.2514 \\
        GraphGPS & Graph Transformer & 0.3153 & 0.3926 & 0.3540 & 0.1270 & 0.0952 & 0.1111 & 0.0272 & 0.8847 & 0.4559 & 0.2611 & 0.5198 & 0.3905 \\
        Graph-JEPA & Pretrained & 0.2571 & 0.2933 & 0.2752 & 0.1308 & 0.1083 & 0.1196 & 0.0165 & 0.8659 & 0.4412 & 0.2180 & 0.5312 & 0.3746 \\
        MDGFM & Pretrained & 0.2187 & 0.2975 & 0.2581 & 0.1188 & 0.1041 & 0.1114 & 0.0112 & 0.7174 & 0.3643 & 0.2285 & 0.5184 & 0.3734 \\
        TGN & Temporal Model & 0.3202 & 0.4010 & 0.3606 & 0.1886 & 0.1112 & 0.1499 & 0.0221 & 0.8776 & 0.4498 & 0.2518 & 0.5691 & 0.4104 \\
        TIDFormer & Temporal Model & 0.3217 & 0.4005 & 0.3611 & 0.2132 & 0.3089 & 0.2611 & 0.0329 & 0.8583 & 0.4456 & 0.2651 & 0.5219 & 0.3935 \\
        \midrule
        \textbf{WorldGraph} & \textbf{GWM} & \textbf{0.3952} & \underline{0.5325} & \underline{0.4640} & \textbf{0.4195} & \underline{0.5395} & \underline{0.4795} & \underline{0.1628} & \textbf{0.9009} & \underline{0.5319} & \underline{0.3841} & \underline{0.6567} & \underline{0.5204} \\
        \textbf{WorldGraph-Pretrain} & \textbf{GWM} & \underline{0.3948} & \textbf{0.5344} & \textbf{0.4646} & \underline{0.4165} & \textbf{0.5996} & \textbf{0.5080} & \textbf{0.1819} & \underline{0.8991} & \textbf{0.5405} & \textbf{0.3890} & \textbf{0.6609} & \textbf{0.5249} \\
        \bottomrule
    \end{tabular}}

    \vspace{2.0mm}
    {\small\bfseries (c) Graph-Level Tasks\par}\vspace{0.8mm}
    \setlength{\tabcolsep}{8.0pt}%
    \renewcommand{\arraystretch}{1.12}%
    \resizebox{\linewidth}{!}{%
    \begin{tabular}{@{\hspace{2pt}}l|@{\hspace{1pt}}l|ccc|ccc|ccc@{\hspace{2pt}}}
        \toprule
        \multirow{2}{*}{Method} & \multirow{2}{*}{Type} & \multicolumn{3}{c|}{Flights} & \multicolumn{3}{c|}{Contact} & \multicolumn{3}{c}{Enron} \\
        & & MAE & RMSE & Macro-F1 & MAE & RMSE & Macro-F1 & MAE & RMSE & Macro-F1 \\
        \midrule
        GCN & MPNN & 0.8047 & 1.2736 & 0.4902 & 0.6842 & 0.9884 & 0.4357 & 1.1674 & 1.7609 & 0.3955 \\
        GAT & MPNN & 0.8354 & 1.3303 & 0.4173 & 0.7470 & 1.0338 & 0.4238 & 1.1887 & 1.8048 & 0.3724 \\
        GraphSAGE & MPNN & 0.8163 & 1.2984 & 0.4763 & 0.6811 & 0.9832 & 0.4402 & 1.1702 & 1.7731 & 0.4052 \\
        SGFormer & Graph Transformer & 0.8043 & 1.2768 & 0.4837 & 0.6823 & 0.9842 & 0.4334 & 1.1533 & 1.7428 & 0.3996 \\
        NodeFormer & Graph Transformer & 0.7941 & 1.2835 & 0.3974 & 0.7554 & 1.0403 & 0.3970 & 1.1573 & 1.7733 & 0.4139 \\
        GraphGPS & Graph Transformer & 0.7723 & 1.2482 & 0.4826 & 0.6378 & 0.9452 & 0.4482 & 1.1713 & 1.7567 & 0.4090 \\
        Graph-JEPA & Pretrained & 0.8407 & 1.3243 & 0.4413 & 0.7109 & 1.0177 & 0.4160 & 1.2211 & 1.7961 & 0.3483 \\
        MDGFM & Pretrained & 0.8440 & 1.3357 & 0.3904 & 0.7385 & 1.0326 & 0.4001 & 1.2257 & 1.8545 & 0.3703 \\
        TGN & Temporal Model & 0.7956 & 1.2615 & 0.4796 & 0.6659 & 0.9713 & 0.4151 & 1.1724 & 1.7767 & 0.4155 \\
        TIDFormer & Temporal Model & 0.7929 & 1.2480 & 0.4759 & 0.6935 & 0.9955 & 0.4250 & 1.1980 & 1.8047 & 0.4102 \\
        \midrule
        \textbf{WorldGraph} & \textbf{GWM} & \underline{0.7069} & \underline{1.1254} & \underline{0.5840} & \underline{0.5648} & \underline{0.8972} & \underline{0.5438} & \underline{1.1103} & \underline{1.5447} & \textbf{0.4255} \\
        \textbf{WorldGraph-Pretrain} & \textbf{GWM} & \textbf{0.6841} & \textbf{1.0985} & \textbf{0.5891} & \textbf{0.5540} & \textbf{0.8865} & \textbf{0.5547} & \textbf{1.0983} & \textbf{1.5383} & \underline{0.4244} \\
        \bottomrule
    \end{tabular}}
    \vspace{-20pt}
\end{table*}
\section{Experiments}
\label{sec:experiments}
\vspace{-10pt}
In this section, we conduct extensive experiments to answer the following research questions:\\
\noindent\textbf{(RQ1):} Compared with existing state-of-the-art methods for solving graph world modeling problems, does WorldGraph have advantages in terms of model quality?
\\
\noindent\textbf{(RQ2):} How effective are the components of WorldGraph?
\\
\noindent\textbf{(RQ3):} How sensitive is WorldGraph to its key model configurations?
\\ 
\noindent\textbf{(RQ4):} 
Is it effective to integrate WorldGraph into the existing graph world model problem settings?
\vspace{-10pt}

\subsection{Experimental Setup}
\label{sec:experimental_settings}
\vspace{-5pt}
\textbf{Benchmark.}
We construct the graph world benchmark, GWM-Zero, at three levels, including node, edge, and graph, to comprehensively evaluate \textit{WorldGraph}. The benchmark is built from TGBN-Trade~\citep{huang2023temporal}, TGBN-Genre~\citep{huang2023temporal}, TGBN-Reddit~\citep{huang2023temporal}, UN Vote~\citep{poursafaei2022towards}, Flights~\citep{poursafaei2022towards}, Contact~\citep{poursafaei2022towards}, SocialEvo~\citep{poursafaei2022towards}, and Enron~\citep{poursafaei2022towards}. \\
\textbf{Baselines.} We select \textbf{13} baselines from four perspectives.
(1) From the graph-representation perspective, we include message passing-based neural networks (MPNN)---GCN~\citep{kipf2017semi}, GAT~\citep{velickovic2018graph}, and GraphSAGE~\citep{hamilton2017inductive}---as well as graph Transformers---SGFormer~\citep{sgformer}, NodeFormer~\citep{wu2022nodeformer}, and GraphGPS~\citep{RampasekGDLWB22}.
(2) From the temporal-modeling perspective, we compare with TGN~\citep{rossi2020temporal} and TIDFormer~\citep{peng2025tidformer}, which explicitly model time-stamped interactions and temporal dependencies.
(3) From the graph-pretraining perspective, we include Graph-JEPA~\citep{skenderi2023graphjepa} and MDGFM~\citep{wang2025mdgfm} to evaluate the benefit of self-supervised and multi-domain  pretraining.
(4) Under existing world-modeling problem settings, we include L$^3$P~\citep{Zhang0S21a}, C-SWM~\citep{KipfPW20}, and GWM-E~\citep{FengWLY25}.
\vspace{-5pt}
\subsection{Overall Performance (\textbf{RQ1})}
\vspace{-5pt}
As shown in Table~\ref{tab:worldgraph_results}, across all 3 task levels, WorldGraph and WorldGraph-Pretrain (Pretraining details are provided in Appendix~\ref{sec:worldgraph_pretraining}.) consistently outperform the baselines.  
\textbf{For node-level tasks},
the gains are particularly pronounced for node deletion on TGBN-Trade, where sparse deletion events cause GraphSAGE and GraphGPS to incorrectly predict the deletion nodes, yielding  $\mathrm{Rem.\  F1}=0$, while WorldGraph-Pretrain reaches Rem. F1 $=0.6606$, an absolute improvement of {50.66\%} over the second best method. 
For semantic-property prediction, WorldGraph-Pretrain improves over the second best methods by {16.24\%} on TGBN-Genre and {13.82\%} on TGBN-Reddit, showing that jointly modeling graph structure, historical states, and graph transitions is beneficial for heterogeneous node transition prediction tasks.
\textbf{For edge-level tasks}, the advantage is more evident, with an average Macro-F1 improvement of 13.72\%. 
{Compared with the strongest baselines on the corresponding dataset, its improvements on UN Vote ($\uparrow$24.69\%), Contact ($\uparrow$8.46\%), and SocialEvo ($\uparrow$11.37\%) demonstrate that the learned transition policy can capture the graph changes that are difficult for representation-only and temporal baselines.}
\textbf{For graph-level tasks}, WorldGraph also obtains the lowest MAE and RMSE and the highest Macro-F1, with average absolute reductions of $0.0757$ and $0.1376$ in MAE and RMSE, respectively, and an average Macro-F1 improvement of $7.18\%$ over the best external baseline for each dataset and metric.
These indicate that its advantage extends from discrete node and edge edits to continuous local-structure evolution.

\vspace{-5pt}
\subsection{Ablation Study (\textbf{RQ2})}
\vspace{-5pt}
{We evaluate five variants that respectively remove the multi-granularity graph encoder, the history-aware state encoder, the transition-aware RL mechanism, as well as its two internal components: dynamic group sampling and structure-aware reward. Figure~\ref{fig:worldgraph_ablation} summarizes their overall performance on Node-Level Tasks/TGBN-Trade and Graph-Level Tasks/Enron.
Removing any component degrades performance on both tasks. The largest reductions are caused by removing the multi-granularity graph encoder or the history-aware state encoder, which lead to absolute drops of $8.83\%$ and $10.20\%$ on Node-Level Tasks/Trade and $2.04\%$ and $2.11\%$ on Graph-Level Tasks/Enron, respectively, confirming that both current graph structure and historical evolution are essential for representing a graph world. Removing the complete RL mechanism, or removing either of its two proposed components---dynamic group sampling or structure-aware reward---also consistently reduces performance, with absolute drops of $6.27/0.85\%$, $4.19/0.55\%$, and $4.84/0.36\%$, respectively. The results show that the two components make complementary contributions to the final performance.}
\begin{figure*}[!t]
    \centering
    \includegraphics[height=3.8cm, width=\linewidth, keepaspectratio]{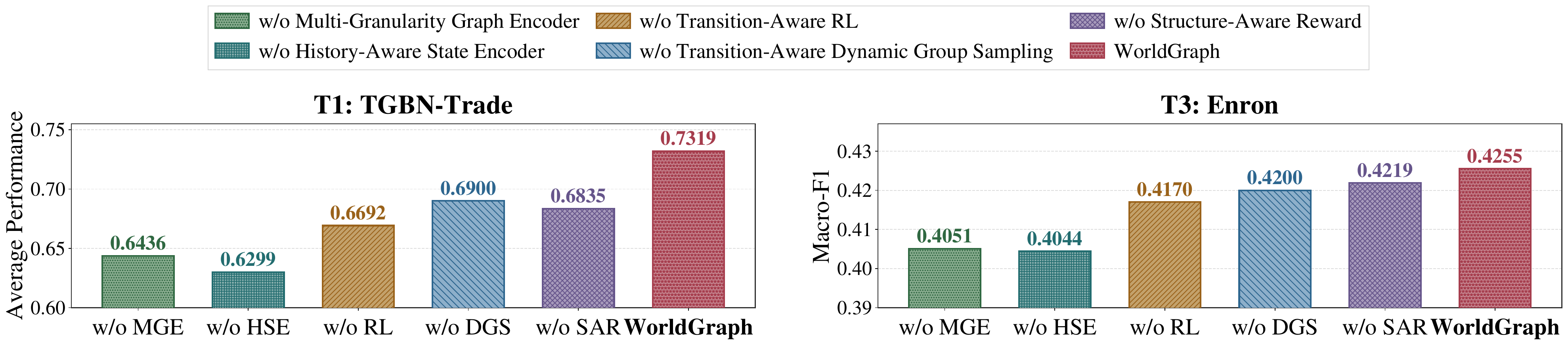}
    \vspace{-25pt}
    \caption{\small \textbf{Ablation Study on Node-Level Tasks/TGBN-Trade and Graph-Level Tasks/Enron.} 
    The left panel reports the equal-weight average of the 4 TGBN-Trade metrics, while the right panel reports Macro-F1.
    }
    \vspace{-10pt}
    \label{fig:worldgraph_ablation}
\end{figure*}
\clearpage

\vspace{-5pt}
\subsection{Sensitivity Analysis (\textbf{RQ3})}
\vspace{-5pt}
\label{sec:sensitivity_analysis}
\begin{wrapfigure}{r}{0.55\textwidth}
\vspace{-10pt}
    \centering
    \includegraphics[height=3.5cm, width=\linewidth, keepaspectratio]{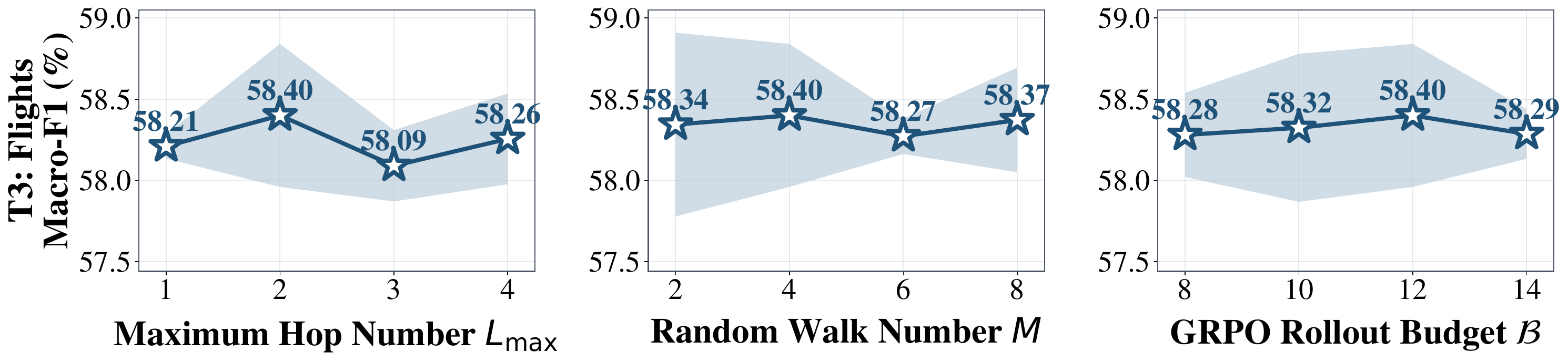}
    \vspace{-20pt}
    \caption{\small \textbf{Sensitivity Analysis on Graph-Level Tasks/Flights.} Performance remains remarkably stable across varying maximum hop numbers \(L_{\max}\), random-walk numbers \(M\), and GRPO rollout budgets \(\mathcal{B}\).}
    \vspace{-10pt}
    \label{fig:worldgraph_sensitivity}
\end{wrapfigure}
Figure~\ref{fig:worldgraph_sensitivity} shows the sensitivity of WorldGraph to three key hyperparameters on Graph-Level Tasks/Flights (Macro-F1). Overall, the performance varies by only 0.31 percentage points across all settings (ranging narrowly from 58.09\% to 58.40\%), demonstrating that WorldGraph is robust to these hyperparameter configurations. Specifically, varying the maximum hop number \(L_{\max}\) from 1 to 4 leads to small fluctuations between 58.09\% and 58.40\%, with \(L_{\max}=2\) achieving the peak performance. The random-walk number \(M\) also exhibits a very marginal impact: the Macro-F1 score remains virtually flat within the range of 58.27\%--58.40\% across \(M \in \{2, 4, 6, 8\}\). For the GRPO rollout budget \(\mathcal{B}\), performance stays stably within 58.28\%--58.40\% without noticeable deviation as \(\mathcal{B}\) increases from 8 to 14. These results confirm that WorldGraph is not sensitive to the choice of these structural and optimization hyperparameters, further validating its stability.
\vspace{-5pt}
\subsection{Comparison on Existing GWM Problem Settings (\textbf{RQ4})}
\vspace{-5pt}
\label{sec:Comparison on Existing GWM Problem Settings}
\begin{wrapfigure}{r}{0.45\textwidth}
\vspace{-10pt}
    \centering
    \includegraphics[width=1\linewidth]{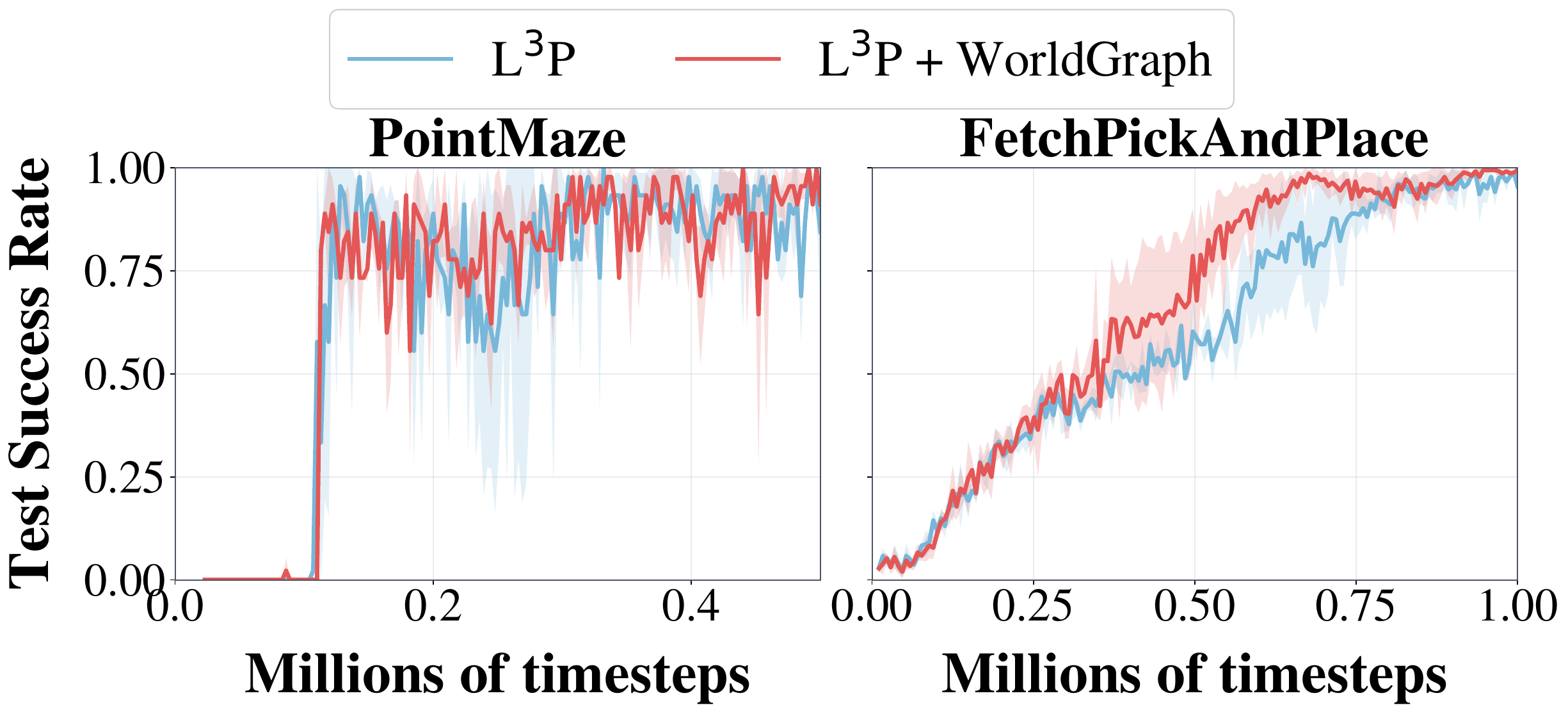}
    \vspace{-20pt}
    \caption{\small \textbf{Comparison with L$^3$P on PointMaze and FetchPickAndPlace.} Solid curves show the mean test success rate and shaded regions show the standard deviation.}
    \label{fig:l3p_fig4_combined}
    \vspace{-20pt}
\end{wrapfigure}
For L$^3$P, we evaluate long-horizon planning on PointMaze~\citep{duan2016benchmarking} and FetchPickAndPlace~\citep{plappert2018multi}. Integrating WorldGraph as a plug-in module provides graph-structured state and history information to the landmark planner while leaving the original environment, actor, and training objective unchanged. The resulting learning curves are shown in Figure~\ref{fig:l3p_fig4_combined}. Over the final ten test epochs, L$^3$P + WorldGraph leads L$^3$P by 7.11 percentage points on PointMaze (94.88\% vs. 87.77\%) and by 2.43 points on FetchPickAndPlace (99.18\% vs. 96.74\%).

\begin{wraptable}{r}{0.6\textwidth}
\vspace{-20pt}
    \centering
    \caption{\small \textbf{Comparison with C-SWM on Multi-Step Latent-State Ranking Tasks.} Higher H@1 and MRR are better.}
    \label{tab:rq4_cswm}
    \scriptsize
    \setlength{\tabcolsep}{5pt}
    \renewcommand{\arraystretch}{1.12}
    \begin{tabular}{ll|cc|cc|cc}
        \toprule
        \multirow{2}{*}{Env.} & \multirow{2}{*}{Model} & \multicolumn{2}{c|}{1 Step} & \multicolumn{2}{c|}{5 Steps} & \multicolumn{2}{c}{10 Steps} \\
        & & H@1 & MRR & H@1 & MRR & H@1 & MRR \\
        \midrule
        \multirow{2}{*}{Pong} & C-SWM & 35.00 & 51.45 & 11.33 & 27.24 & 6.67 & 19.03 \\
        & \textbf{+ WorldGraph} & \textbf{40.33} & \textbf{57.68} & \textbf{26.33} & \textbf{45.85} & \textbf{16.00} & \textbf{34.62} \\
        \midrule
        \multirow{2}{*}{SI} & C-SWM & 63.67 & 75.49 & 41.33 & 56.66 & 29.67 & 46.16 \\
        & \textbf{+ WorldGraph} & \textbf{68.00} & \textbf{79.68} & \textbf{62.67} & \textbf{76.04} & \textbf{61.00} & \textbf{75.71} \\
        \bottomrule
    \end{tabular}
    \caption{\small \textbf{Comparison with GWM-E on Node Classification (NC), Link Prediction (LP), and Graph Classification (GC) tasks (Accuracy/\%).} Higher accuracy is better.}
    \label{tab:rq4_gwme}
    \scriptsize
    \setlength{\tabcolsep}{7pt}
    \renewcommand{\arraystretch}{1.12}
    \begin{tabular}{lccccc}
        \toprule
        \textbf{Model} & \multicolumn{2}{c}{\textbf{Cora}} & \multicolumn{2}{c}{\textbf{PubMed}} & \textbf{HIV} \\
        & NC& LP & NC & LP & GC \\
        \midrule
        GWM-E & 83.03 & 94.31 & 84.22 & 94.01 & 93.86 \\
        \textbf{WorldGraph} & \textbf{88.93} & \textbf{95.17} & \textbf{89.34} & \textbf{97.04} & \textbf{96.49} \\
        \bottomrule
    \end{tabular}
    \vspace{-20pt}
\end{wraptable}
For C-SWM, we evaluate multi-step latent-state ranking on Pong~\citep{brockman2016openai} and Space Invaders (SI)~\citep{brockman2016openai} environments (env.) at 1, 5, and 10 steps. The plug-in augments C-SWM's object-centric transition model with graph-aware state and transition information, while retaining its original object encoder, transition backbone, and contrastive training objective. As shown in Table~\ref{tab:rq4_cswm}, WorldGraph improves both H@1 and MRR at every horizon on both games. The gains become larger at longer horizons, reaching 31.33 in H@1 and 29.55  in MRR at 10 steps on SI.

For comparison with GWM-E, we use its traditional graph prediction setting on Cora~\citep{mccallum2000automating}, PubMed~\citep{namata2012query}, and HIV~\citep{wu2018moleculenet}, covering node classification, link prediction, and graph classification. We independently train WorldGraph on these tasks with its multi-granularity graph encoder and task-specific prediction heads. WorldGraph achieves higher accuracy in all five columns, with absolute gains of 5.90 and 0.86 on Cora node classification and link prediction, 5.12 and 3.03 on the corresponding PubMed tasks, and 2.63 on HIV graph classification.

\section{Conclusion}
In this paper, we study \emph{graph world modeling} by formalizing latent world states and heterogeneous graph transitions from evolving graph observations.
We introduce \textit{GWM-Zero} to cover node-, edge-, and graph-level transition prediction across eight temporal graph datasets, and propose \textit{WorldGraph}, which uses a state-aware graph transformer for multi-granularity state modeling and a transition-aware GRPO with dynamic grouping and structure-aware verifiable rewards for learning under frequency imbalance.
Extensive experiments show that WorldGraph consistently achieves stronger performance across diverse \emph{multi-granularity graph transition prediction tasks}. WorldGraph provides a general and effective framework for graph-native world modeling and opens up promising directions for learning richer latent dynamics and more controllable graph evolution in future work.

\bibliography{iclr2027_conference}
\bibliographystyle{iclr2027_conference}

\newpage
\appendix
\tableofcontents
\newpage

\section{Notations}
\label{apsec:notations}
We provide the detailed notation in Table~\ref{tab:notation}.

\begin{table*}[!h]
\centering
\footnotesize
\caption{Core notations used in WorldGraph.}
\label{tab:notation}
\renewcommand{\arraystretch}{0.94}
\setlength{\tabcolsep}{4pt}
\begin{tabular}{p{0.27\linewidth}p{0.67\linewidth}}
\toprule
\textbf{Symbol} & \textbf{Description} \\
\midrule
\multicolumn{2}{l}{\bf \textit{Graph World}} \\
$t,\ \tau,\ T$ & Current time, historical time, and sequence length. \\
$g_t=(V_t,E_t,X_t)\in\mathcal{G}$ & Observed graph at time $t$, its node, edge, and property sets, and the graph space. \\
$\Delta g_t$ & Graph change associated with the transition from $g_t$ to $g_{t+1}$. \\
$\mathbf{s}_t\in\mathcal{S},\ f_s$ & Latent world state, world-state space, and state model. \\
$f_\Delta$ & Transition model that predicts $\Delta g_t$ from the current graph, the preceding graph change, and $\mathbf{s}_t$. \\
$\mathbf{z}_{v,t},\ \mathbf{z}_t,\ \mathbf{s}_{v,t}$ & Structure-aware node embedding, graph embedding, and node-level latent state at time $t$. \\
\addlinespace
\multicolumn{2}{l}{\bf\textit{State-Aware Graph Transformer}} \\
$L_{\max},\ \ell$ & Maximum hop number and hop index. \\
$\mathbf{x}_v^{(\ell)},\ \mathbf{x}_v^{\mathrm{path}_i}$ & Hop-level and $i$-th random-walk-path embeddings of node $v$. \\
$M,\ \mathrm{path}_i$ & Number of sampled random-walk paths and the $i$-th path. \\
$\mathbf{x}_{\Delta g_\tau},\ \mathrm{Type}(\Delta g_\tau)$ & Transition representation and learnable transition-type embedding for $\Delta g_\tau$. \\
$\mathbf{m}_{v,\tau},\ \mathrm{Dist}(t,\tau),\ t_\mathrm{max}$ & Historical memory of node $v$, relative temporal-distance embedding, and maximum historical time span. \\
$\mathbf{W}_x,\ \mathbf{W}_z,\ \mathbf{W}_a$ & Learnable projections of node history, graph history, and transition-type information in the node memory. \\
$\mathbf{q}_{v,t},\ \beta_{v,t,\tau},\ \tilde{\beta}_{v,t,\tau},\ d$ & History-attention query, attention score, normalized attention weight, and projected embedding dimension. \\
$\mathbf{p}_{t-\tau},\ \Pi(\Delta g_\tau,\Delta g_{t-1})$ & Relative-time attention bias and similarity between a historical transition and the latest observed transition. \\
$\mathbf{W}_Q,\ \mathbf{W}_K,\ \mathbf{W}_V$ & Learnable query, key, and value projections for history attention. \\
$\mathbf{h}_{v,t},\ \tilde{\mathbf{h}}_{v,t}$ & Single-head history embedding and the final multi-head history embedding of node $v$. \\
$g^{(\mathrm{SG})},\ g_i^{(\mathrm{OURS})}$ & An SGFormer subgraph and a WorldGraph sampled subgraph. \\
$g_i^{(\mathrm{IDEAL})},\ \zeta$ & An ideal full-coverage subgraph and the ideal coverage count. \\
$\mathbf{s}_{v,t}^{(\mathrm{SG})},\ \mathbf{s}_{v,t}^{(\mathrm{OURS})}$ & Final latent states induced by SGFormer and WorldGraph. \\
$\mathbf{s}_{v,t}^{(\mathrm{IDEAL})}$ & Final latent state induced by the ideal full-coverage encoder. \\
$\Delta_{\mathrm{SG}},\ \Delta_{\mathrm{OURS}}$ & Deviations of SGFormer and WorldGraph from the ideal latent state. \\
\addlinespace
\multicolumn{2}{l}{\textit{\bf Transition-Aware RL Algorithm}} \\
$\mathcal{J}_t,\ J_i,\ |J_i|$ & Set of transition-conditioned groups at time $t$, its $i$-th group, and the group size. \\
$\widehat{p}(J_i),\ \omega(J_i),\ \gamma$ & Historical transition frequency, rarity weight, and rarity-weight exponent of group $J_i$. \\
$\mathrm{Prob}(J_i),\ \mathcal{B},\ B_t(J_i)$ & Selection probability, total rollout budget, and budget allocated to group $J_i$. \\
$\rho^{\mathrm{GRPO}},\ \rho^{\mathrm{OURS}},\ \rho_i,\ C_a$ & Group sizes under standard GRPO and Transition-aware RL, a generic group size, and the base group-size constant. \\
$\mathrm{EVR}(\rho_i),\ M_{\mathrm{EVR}},\ \delta$ & Effective variance ratio, its residual term, and reward-distribution kurtosis. \\
\addlinespace
$\mathcal{I}_{\mathbf{s}}(v),\ \deg(v),\ \theta$ & Degree-based structural importance, node degree, and degree-scaling exponent. \\
$f_{<t}(v),\ \mathcal{I}_{\mathbf{v}}(v),\ \eta$ & Historical change frequency, variability importance, and variability-scaling exponent. \\
$\mathcal{I}(v),\ V_t$ & Normalized importance of node $v$ and the current node set over which it is normalized. \\
$\Delta\hat{Q},\ \Delta Q$ & Predicted and ground-truth change sets. \\
$\mathcal{P},\ \mathcal{R},\ R_{\mathrm{structure}}$ & Importance-weighted precision, recall, and their harmonic-mean structure reward. \\
$\lambda,\ R_{\mathrm{property}},\ R$ & Structural/property reward balancing coefficient, property-change reward, and unified task reward. \\
\bottomrule
\end{tabular}
\end{table*}
\newpage

\section{Additional Technical Details}
\subsection{Detailed Design of State-aware Graph Transformer}

{\subsubsection{Details of Multi-granularity Graph Encoder}
\textbf{Details of Hop-Level Message Passing.}
Given the maximum hop number $L_{\max}$, for a given graph $g_t$, we perform message passing for $\ell = \{1, 2, \cdots, L_{\max}\}$ over the neighbors of each node to obtain the initial node embeddings.
Let $\mathcal{N}(v)$ denote the neighbor set of node $v$. We compute the node embedding of $v$:
\begin{equation}
    \mathbf{x}_v^{(\ell)} = \operatorname{ReLU}\!\left(
    \mathbf{W}^{(\ell)}_\mathrm{self}\mathbf{x}_v^{(\ell-1)} +
    \frac{
    \sum_{u\in\mathcal{N}(v)} w_{uv}\,
    \mathbf{W}_{\mathrm{neighbor}}^{(\ell)}\mathbf{x}_u^{(\ell-1)}
    }{
    \max\!\left(1,\sum_{u\in\mathcal{N}(v)}w_{uv}\right)
    }
    \right)
\end{equation}
, where $w_{uv}$ is set to the currently observed edge weight for a weighted graph, and to $1$ for an unweighted graph, in which case the neighborhood term reduces to mean aggregation. $\mathbf{W}_\mathrm{self}^{(\ell)}$ and $\mathbf{W}_{\mathrm{neighbor}}^{(\ell)}$ denote the learnable weight matrices at layer $\ell$. Thus, we obtain node embeddings computed through multi-layer message passing based on neighbors at different hop distances and, when available, their observed relation strengths.

\textbf{Details of Random-Walk Path Sampling.}
We use projected scaled dot-product attention to aggregate the node embeddings along $\text{path}_i$. Specifically, for each node $v_j \in \text{path}_i$, we compute the attention score with respect to the starting node $v_0$ as
\begin{equation}
  \mathrm{Score}_{i,j}=
\frac{
\left(\mathbf{W}_{\mathrm{path}}\mathbf{x}_{v_j}^{(L_{\max})}\right)^\top
\left(\mathbf{W}_{\mathrm{path}}\mathbf{x}_{v_0}^{(L_{\max})}\right)
}{\sqrt{d}},  
\end{equation}
where $\mathbf{W}_{\mathrm{path}}$ is a learnable projection for path attention and $d$ is the dimension of the projected node embedding. We normalize the scores over the nodes on the same path: 
\begin{equation}
    \alpha_{i,j}=
\frac{\exp(\mathrm{Score}_{i,j})}
{\sum_{k=0}^{T_i}\exp(\mathrm{Score}_{i,k})}.
\end{equation}
The embedding corresponding to the $i$-th path is then obtained by weighted aggregation: 
\begin{equation}
    \mathbf{x}_v^{\text{path}_i}
=\sum_{j=0}^{T_i}\alpha_{i,j}\,\mathbf{x}_{v_j}^{(L_{\max})}.
\end{equation}

Thus, we derive \(\mathbf{x}_v^{\text{path}_i}\) for different sampled paths \(\text{path}_i\), capturing multi-step structural information around node \(v\) from different random-walk trajectories.

After obtaining the hop-level embeddings $\{\mathbf{x}_v^{(\ell)}\}_{\ell=1}^{L_{\max}}$ and the path-level embeddings $\{\mathbf{x}_v^{\mathrm{path}_i}\}_{i=1}^{M}$, we fuse all structural representations of node $v$ via an attention-based aggregation to obtain a structure-aware embedding $\mathbf{z}_{v,t}$.

\textbf{Details of Multi-Granularity Attention Fusion.}
Then, we collect all candidate representations into a unified set $\mathcal{S}_v=\big\{\mathbf{x}_v^{(1)},\ldots,\mathbf{x}_v^{(L_{\max})},\mathbf{x}_v^{\mathrm{path}_1},\ldots,\mathbf{x}_v^{\mathrm{path}_M}\big\}$.
Let $\{\mathbf{u}_{v,k}\}_{k=1}^{L_{\max}+M}$ denote the candidates in $\mathcal{S}_v$, where $\mathbf{u}_{v,k}=\mathbf{x}_v^{(k)}$ for $1\le k\le L_{\max}$ and $\mathbf{u}_{v,k}=\mathbf{x}_v^{\mathrm{path}_{k-L_{\max}}}$ for $L_{\max}< k\le L_{\max}+M~$.
We compute an unnormalized attention score between each candidate $\mathbf{u}_{v,k}$ and the last-hop embedding of node $v$ as
$\alpha_{v,k}=
\frac{
\left(\mathbf{W}_{\mathrm{fusion}}\mathbf{u}_{v,k}\right)^\top
\left(\mathbf{W}_{\mathrm{fusion}}\mathbf{x}_v^{(L_{\max})}\right)
}{\sqrt{d}}$,
where $\mathbf{W}_{\mathrm{fusion}}$ is a learnable projection for fusing the hop- and path-level representations. We then normalize the scores over all candidates via a softmax:
$\tilde{\alpha}_{v,k}=\frac{\exp(\alpha_{v,k})}{\sum_{k'=1}^{L_{\max}+M}\exp(\alpha_{v,k'})}$.
The final structure-aware embedding is retained as the node-level world state: 
\begin{equation}
\label{eq:worldstate}    
\mathbf{z}_{v,t}=\sum_{k=1}^{L_{\max}+M}\tilde{\alpha}_{v,k}\,\mathbf{u}_{v,k}.
\end{equation}

\subsubsection{Details of History-aware State Encoder}

\textbf{Details of Query Embedding.} For each node $v$, we construct a query embedding from its latest node-level world state, the global graph summary, the latest transition, and its previous latent state:
$\mathbf{q}_{v,t}
= [\mathbf{z}_{v,t};\,\mathbf{z}_t;\,
\mathbf{x}_{\Delta g_{t-1}};\,\mathbf{s}_{v,t-1}]$,
and aggregate its historical memory $\{\mathbf{m}_{v,\tau}\}$ with multi-head attention.

\textbf{Details of Multi-Head Attention} 
The same operation is performed in parallel by all attention heads. For one attention head, the history embedding is obtained by weighted aggregation of the value projections:

 \begin{equation}
 \mathbf{h}_{v,t}
 =\sum_{\tau}
 \widetilde{\beta}_{v,t,\tau}
 \left(\mathbf{m}_{v,\tau}\mathbf{W}_{V}\right).
\end{equation}
The outputs from all heads are concatenated and mixed by an MLP to obtain the final history embedding:
\begin{equation}
    \widetilde{\mathbf{h}}_{v,t}
    =\operatorname{MLP}\!\left(
    \big[\mathbf{h}_{v,t}^{(1)};\mathbf{h}_{v,t}^{(2)};\ldots;\mathbf{h}_{v,t}^{(H)}\big]
    \right).
\end{equation}

\subsection{Detailed Design of Transition-aware RL Algorithm}
\subsubsection{Details of Transition-Aware Dynamic Group Sampling}
\paragraph{Details of Rarity-based Rollout Budget Allocation}
Finally, to make earlier rare transitions receive a larger effective group size, we allocate the total rollout budget $\mathcal{B}$ across groups proportionally. 
We reserve the same minimum number of rollouts\footnote{In practice, we set the minimum number of rollouts as 2. We do this to ensure every transition group receives at least a minimal number of rollouts, preventing overly noisy return estimates from groups with too few samples.} for every group and use largest-remainder rounding for the remaining budget, such that $\sum_{j\in\mathcal{J}_t}B_t(j)=\mathcal{B}$, where $B_t(j)$ is the allocated number of rollouts for transition group $j$ at time step $t$.
Thus, groups corresponding to transitions that are rarer in the prefix $\{\Delta g_\tau\}_{\tau<t}$ are assigned more rollout samples at time $t$, increasing the relative training signal for these frequency-imbalanced modification patterns.

\subsubsection{Structure-aware Reward Design}

\textbf{Details of Node/Edge/Graph Importance Calculation.}
The three graph-world tasks produce outputs at different granularities. We therefore retain the same structural importance $ \mathcal{I}(v)$, while defining a verifiable reward consistent with the output and evaluation metric of each task. For a changed edge $e=(u,v)$, its importance is determined by its two nodes: $\mathcal{I}(e)=\frac{ \mathcal{I}(u)+ \mathcal{I}(v)}{2}$. For a changed local graph centered at node $v$, its importance is $\mathcal{I}(v)$.

Although the three graph-world tasks operate at different granularities, their rewards share the same principle: a prediction should identify the correct structural changes and estimate the corresponding representation changes. We therefore use the unified reward
\begin{equation}
R=\lambda R_{\mathrm{structure}}+(1-\lambda)R_{\mathrm{property}},
\end{equation}
where $R_{\mathrm{property}}$ is calculated based on the ground-truth updated property and the predicted property. $\lambda$ balances the two terms.

\subsection{Details of Task-specific Structure-aware Reward Calculation}
\label{sec:task_specific_reward_details}

\textbf{Node-level tasks.}
Let $V_{t+1}^{\mathrm{add}}$, $V_{t+1}^{\mathrm{delete}}$, and $\Delta V_{t+1}^{\mathrm{property}}$ denote the ground-truth sets of added, deleted, and property-changed nodes, respectively; their hatted versions denote the corresponding predictions. For a property-changed node $v$, $\Delta\mathbf{x}_{v,t+1}$ and $\widehat{\Delta\mathbf{x}}_{v,t+1}$ denote its ground-truth and predicted property changes. The importance-weighted property reward is
\begin{equation}
R_{\mathrm{property}}^{\mathrm{node}}=\exp\!\left(
-\frac{\sum_{v\in\Delta V_{t+1}^{\mathrm{property}}}\mathcal{I}(v)\,
\operatorname{MAE}\!\left(\widehat{\Delta\mathbf{x}}_{v,t+1},\Delta\mathbf{x}_{v,t+1}\right)}
{\sum_{v\in\Delta V_{t+1}^{\mathrm{property}}}\mathcal{I}(v)}
\right).
\end{equation}
Here, $R_{\mathrm{structure}}^{(\mathrm{node})}$ averages the set rewards for node addition, node deletion, and property-change localization. Setting $\lambda_{\mathrm{node}}=3/4$ in the unified reward gives equal weight to the four constituent terms:
\begin{equation}
\begin{aligned}
R_{\mathrm{node}}=\frac{1}{4}\big[&
R_{\mathrm{structure}}\!\left(\widehat{V}_{t+1}^{\mathrm{add}},V_{t+1}^{\mathrm{add}}\right)
+R_{\mathrm{structure}}\!\left(\widehat{V}_{t+1}^{\mathrm{delete}},V_{t+1}^{\mathrm{delete}}\right)\\
&+R_{\mathrm{structure}}\!\left(\widehat{\Delta V}_{t+1}^{\mathrm{property}},\Delta V_{t+1}^{\mathrm{property}}\right)
+R_{\mathrm{property}}^{\mathrm{node}}\big].
\end{aligned}
\end{equation}

\textbf{Edge-level tasks.}
Let $E_{t+1}^{\mathrm{add}}$, $E_{t+1}^{\mathrm{delete}}$, and $\Delta E_{t+1}^{\mathrm{property}}$ denote the ground-truth sets of added, deleted, and property-changed edges, respectively; their hatted versions denote the corresponding predictions. For a property-changed edge $e$, $\Delta\mathbf{x}_{e,t+1}$ and $\widehat{\Delta\mathbf{x}}_{e,t+1}$ denote its ground-truth and predicted property changes. The importance-weighted property reward is
\begin{equation}
R_{\mathrm{property}}^{\mathrm{edge}}=\exp\!\left(
-\frac{\sum_{e\in\Delta E_{t+1}^{\mathrm{property}}}\mathcal{I}(e)\,
\operatorname{MAE}\!\left(\widehat{\Delta\mathbf{x}}_{e,t+1},\Delta\mathbf{x}_{e,t+1}\right)}
{\sum_{e\in\Delta E_{t+1}^{\mathrm{property}}}\mathcal{I}(e)}
\right).
\end{equation}
Here, $R_{\mathrm{structure}}^{(\mathrm{edge})}$ averages the set rewards for edge addition, edge deletion, and property-change localization. With $\lambda_{\mathrm{edge}}=3/4$, the edge-level reward is
\begin{equation}
\begin{aligned}
R_{\mathrm{edge}}=\frac{1}{4}\big[&
R_{\mathrm{structure}}\!\left(\widehat{E}_{t+1}^{\mathrm{add}},E_{t+1}^{\mathrm{add}}\right)
+R_{\mathrm{structure}}\!\left(\widehat{E}_{t+1}^{\mathrm{delete}},E_{t+1}^{\mathrm{delete}}\right)\\
&+R_{\mathrm{structure}}\!\left(\widehat{\Delta E}_{t+1}^{\mathrm{property}},\Delta E_{t+1}^{\mathrm{property}}\right)
+R_{\mathrm{property}}^{\mathrm{edge}}\big].
\end{aligned}
\end{equation}

\textbf{Graph-level tasks.}
Let $\Delta V_{t+1}^{\mathrm{graph}}$ and $\widehat{\Delta V}_{t+1}^{\mathrm{graph}}$ denote the ground-truth and predicted sets of nodes whose local graphs change. The local-graph representation $\boldsymbol{\phi}_{v,t}$ contains node count, edge count, density, number of connected components, clustering coefficient, and triangle count. We use $\Delta\boldsymbol{\phi}_{v,t+1}=\boldsymbol{\phi}_{v,t+1}-\boldsymbol{\phi}_{v,t}$ and $\widehat{\Delta\boldsymbol{\phi}}_{v,t+1}=\widehat{\boldsymbol{\phi}}_{v,t+1}-\boldsymbol{\phi}_{v,t}$ for the ground-truth and predicted representation changes. When evaluating a sampled transition, the property-change magnitude is computed over its selected local-graph centres that also exhibit a ground-truth structural change. Their property reward is

\begin{equation}
R_{\mathrm{property}}^{\mathrm{graph}}=\exp\!\left(
-\frac{\sum_{v\in\Delta V_{t+1}^{\mathrm{graph}}}\mathcal{I}(v)\,
\operatorname{MAE}\!\left(\widehat{\Delta\boldsymbol{\phi}}_{v,t+1},\Delta\boldsymbol{\phi}_{v,t+1}\right)}
{\sum_{v\in\Delta V_{t+1}^{\mathrm{graph}}}\mathcal{I}(v)}
\right).
\end{equation}
With $\lambda_{\mathrm{graph}}=1/2$, the graph-level reward equally combines change localization and representation-change magnitude:
\begin{equation}
R_{\mathrm{graph}}=\frac{1}{2}\big[
R_{\mathrm{structure}}\!\left(\widehat{\Delta V}_{t+1}^{\mathrm{graph}},\Delta V_{t+1}^{\mathrm{graph}}\right)
+R_{\mathrm{property}}^{\mathrm{graph}}\big].
\end{equation}
In practice, we add a small number (e.g., $10^{-8}$) when computing the reward to avoid the numerical instability caused by a zero denominator.

\subsection{Pseudo-code of Transition-aware RL}
\label{sec:ap_transition_aware_rl}
Algorithm~\ref{alg:transition_aware_rl} summarizes the training and inference procedure of the proposed transition-aware reinforcement learning algorithm. At each training step, the input is the current graph $g_t$, its observed historical transitions and latent states, and the next graph $g_{t+1}$ used only as supervision. The state model first encodes $g_t$ and its history into $\mathbf{s}_t$. From transitions observed before $t$, we then construct transition-conditioned groups $\mathcal{J}_t$, with one group corresponding to each transition type observed in the history. We estimate the historical frequency $\widehat{p}_t(J_i)$ of each group, convert it to the rarity weight $\omega_t(J_i)=\widehat{p}_t(J_i)^{-\gamma}$, and normalize these weights to obtain the group allocation probabilities. Every group receives the same base number $C_a$ of rollouts, while the remaining budget is allocated according to these probabilities, giving relatively more samples to transition types that are less frequently observed.

For each allocated rollout, the first operation is constrained to the corresponding group; for example, in a node-level task, one group may start with node addition, another with node deletion, and another with a node feature change. The Controller then samples subsequent valid operations and target nodes or edges until it stops. The resulting operation sequence is decoded as a predicted change set $\Delta\hat{Q}$ (together with the task-specific representation change). Comparing it with the ground-truth change set $\Delta Q$ and representation change from $g_{t+1}$ gives the structure reward $R_{\mathrm{structure}}$ and property reward $R_{\mathrm{property}}$, which are combined into the scalar return $R=\lambda R_{\mathrm{structure}}+(1-\lambda)R_{\mathrm{property}}$. We use GRPO to compare these returns only among rollouts that began with the same edit type. Its core idea is to avoid fitting a separate value or critic network: rewards are normalized within each transition-conditioned group to form group-relative advantages. A rollout with an above-group-average return receives a positive advantage, whereas a below-average rollout receives a negative one; the clipped policy-ratio objective then increases the probability of the former and decreases the probability of the latter. At inference time, the trained Controller supplies the transition conditioning and the task-specific prediction heads decode the final output from the predicted future latent state.

\begin{algorithm}[!t]
\caption{Transition-aware RL Training and Inference}
\label{alg:transition_aware_rl}
\begin{algorithmic}[1]
\Require Chronological graph sequence $\{g_t\}_{t=1}^{T}$, state model $f_s$, transition Controller $f_\Delta$, rollout budget $\mathcal{B}$, base group size $C_a$, exponents $\gamma,\theta,\eta$, and reward coefficient $\lambda$
\Ensure Trained WorldGraph model

\State Initialize the state model, transition Controller, and task-specific prediction heads

\For{$t=1,2,\ldots,T-1$}
    \State Encode the current graph and its history to obtain $\mathbf{s}_t=f_s(g_t,\Delta g_{t-1},\mathbf{s}_{t-1})$
    \State Compute the degree-based and historical-variability importance terms and obtain $\mathcal{I}(v)$ using $\theta$ and $\eta$
    \State Construct the transition-conditioned groups $\mathcal{J}_t$ from transitions observed before time $t$
    \For{each group $J_i\in\mathcal{J}_t$}
        \State Estimate its historical frequency $\widehat{p}(J_i)$ and compute $\omega(J_i)=\widehat{p}(J_i)^{-\gamma}$
    \EndFor
    \State Compute $\mathrm{Prob}(J_i)$ by normalizing the rarity weights over $\mathcal{J}_t$
    \State Allocate $B_t(J_i)$ rollouts to each group, reserving $C_a$ rollouts per group and distributing the remaining budget according to $\mathrm{Prob}(J_i)$
    \For{each group $J_i\in\mathcal{J}_t$}
        \For{$b=1,2,\ldots,B_t(J_i)$}
            \State Force the first operation to belong to $J_i$
            \State Sample subsequent valid operations and target nodes or edges from $f_\Delta$ until the stopping decision
            \State Decode the sampled operation sequence as a predicted change set $\Delta\hat{Q}$
            \State Compare $\Delta\hat{Q}$ with the ground-truth change set $\Delta Q$ and compute $R_{\mathrm{structure}}$
            \State Compute the task-specific property reward $R_{\mathrm{property}}$
            \State Compute $R=\lambda R_{\mathrm{structure}}+(1-\lambda)R_{\mathrm{property}}$
        \EndFor
    \EndFor
    \State Normalize the rewards within each transition-conditioned group to obtain group-relative advantages
    \State Update $f_\Delta$ with the clipped GRPO objective using all sampled rollouts and their group-relative advantages
\EndFor

\State At inference time, use the trained Controller to provide the transition conditioning, and use the task-specific prediction heads to decode the final task output from the predicted future latent state
\State \Return the trained WorldGraph model
\end{algorithmic}
\end{algorithm}

\clearpage

\section{Proof of Theorems}  
\label{sec:ap-c}
\subsection{Proof of Theorem~\ref{thm:graphtransformer}}
\label{sec:ap-c-thm1}
\begin{proof}
From Equation~\ref{eq:finalstate}, the final latent state can be written as
\begin{equation}
\mathbf{s}_{v,t}^{(\mathrm{method})} = \mathcal{F}\!\left(\mathbf{z}_{v,t}^{(\mathrm{method})}\right).
\end{equation}

$\mathcal{F}$ is a deterministic function  induced by the structure and history information with the same $\Delta g_{t-1}$, $\mathbf{s}_{v, t-1}$\footnote{Since SGFormer is designed for static graph scenarios, we use the same $\mathbf{s}_{t-1}$ to mitigate the effects brought by the quality of temporal encodings, thereby allowing the model to focus more on graph-structure encoding.}, and model parameters.

Following existing theoretical analysis of graph learning methods~\citep{HuangL0YZJ024,maskey2026graph,JiaZV24}, we consider that the message function $\mathcal{F}(\cdot)$ is $C$-Lipschitz.
Then the deviation between a method and the ideal model is bounded by the induced deviation of the aggregated inputs:
\[
\Delta_{\mathrm{method}} = 
\left\|\mathbf{s}_{v,t}^{(\mathrm{method})}-\mathbf{s}_{v,t}^{(\mathrm{IDEAL})}\right\|_2
\le C_L\cdot \left\|\mathbf{z}_{v,t}^{(\mathrm{method})}-\mathbf{z}_{v,t}^{(\mathrm{IDEAL})}\right\|_2
\]
where $C_L$ is a constant.

Based on Equation~\ref{eq:worldstate}, we can model $\mathbf{z}_{v,t}^{(\mathrm{method})} = \sum_{i}\psi(g_i)$. In our state-aware graph transformer (SGT), $\psi(g_i) = \tilde{\alpha}_i\mathbf{u}_{v,k}$, as defined in Equation~\ref{eq:worldstate}. 
To avoid representation drift of the same subgraph instance induced by different hyperparameters in attention designs (e.g., the number of attention heads),
we use a consistent attention encoder across SGFormer, SGT (ours), and the ideal model.
Consequently, for any training subgraph $g_i$,
$\psi_{\mathrm{SGFormer}}(g_i)=\psi_{\mathrm{OURS}}(g_i)=\psi_{\mathrm{IDEAL}}(g_i)=\psi(g_i)$.
Following~\citep{wan2022bns}, we can bound $\psi(g_i)$ by a constant $C_p$. $\tilde{\alpha}_i$ is bounded by $1$ because it is exponentially normalized.

Therefore, $\left\|\mathbf{z}^{(\mathrm{method})}-\mathbf{z}^{(\mathrm{IDEAL})}\right\|_2 = \|\sum_{g_i \notin g^{(\mathrm{method})}}\psi(g_i)\|_2 \leq \|\sum_{g_i \notin g^{(\mathrm{method})}}C_p\|_2$. 
Then, 
the upper bound of $\Delta_{\mathrm{OURS}}=\left\|\mathbf{z}^{(\mathrm{OURS})}-\mathbf{z}^{(\mathrm{IDEAL})}\right\|_2$ is $(\zeta - (L_{\max} + M))\|C_p\|_2$. The upper bound of $\Delta_{\mathrm{SGFormer}}=\left\|\mathbf{z}^{(\mathrm{SGFormer})}-\mathbf{z}^{(\mathrm{IDEAL})}\right\|_2$ is $(\zeta - 1)\|C_p\|_2 > (\zeta - (L_{\max} + M))\|C_p\|_2$.
Therefore, the upper bound of \(\Delta_{\mathrm{OURS}}\) is smaller than that of \(\Delta_{\mathrm{SG}}\).

\end{proof}

\subsection{Proof of Theorem~\ref{thm:rl}}
\label{sec:ap-c-thm2}
\begin{proof}
Because GRPO uses a fixed group size, the group size corresponding to each transition type is a constant \(C_a\). Hence, 
\begin{equation}
\rho^{\mathrm{GRPO}} = C_a.    
\end{equation}
Our RL algorithm uses a dynamic group size based on the transition frequencies. As discussed in Footnote~\ref{ft:minmumrollout}, we assign a minimum number of rollouts to different transition types. In this case, each transition type corresponds to the same group size as in GRPO. Then, based on the transition frequency, we allocate additional rollout budget to transitions with larger/minimum counts (i.e., larger group sizes) for different transition types, which leads to
\begin{equation}
\rho^{\mathrm{OURS}} = C_a + \mathrm{Prob}(J_i)\cdot \mathcal{B} > \rho^{\mathrm{GRPO}} = C_a.
\end{equation}

From~\citep{kim2026efficiency}, $\mathrm{EVR}(\rho)=1+\frac{\delta-1}{2\rho}+M_\mathrm{EVR}$.
Following~\citep{kim2026efficiency}, we consider $\delta>1$. Since $\mathrm{EVR}(\rho)$ is monotonically decreasing with respect to $\rho$, 
\begin{equation}
    \mathrm{EVR}(\rho^{\mathrm{OURS}}) < \mathrm{EVR}(\rho^{\mathrm{GRPO}}).
\end{equation}
This completes the proof.

\end{proof}
\newpage

\section{Experimental Validation of Theoretical Results}
\subsection{Validation of the Graph-Encoder Bound}
\label{sec:ap-c-thm1-exp}
To empirically verify the ordering predicted by Theorem~\ref{thm:graphtransformer}, we use five frozen WorldGraph checkpoints on Graph-level tasks/Flights and evaluate the same test transitions without further training. We construct a full-view reference representation using an expanded deterministic structural-view set, and compare it with (i) SGFormer, which retains only one structural view, and (ii) the default multi-view SGT encoder. We report the mean node-wise Euclidean distances to the reference representation and the corresponding hidden-state distances. As shown in Table~\ref{tab:sgt_theorem_verification}, the default multi-view SGT has substantially smaller deviations than SGFormer for both $\mathbf{z}_{v,t}$ and $\mathbf{s}_{v,t}$, consistent with the smaller deviation bound established in Appendix~\ref{sec:ap-c-thm1}.

\begin{table}[!t]
    \centering
    \small
    \caption{Experimental validation of Theorem~\ref{thm:graphtransformer} on Graph-level tasks/Flights. Values are mean $\pm$ standard deviation over five checkpoints; lower is better.}
    \label{tab:sgt_theorem_verification}
    \setlength{\tabcolsep}{6pt}
    \renewcommand{\arraystretch}{1.08}
    \begin{tabular}{lcc}
        \toprule
        Encoder setting & $\|\mathbf{z}_{v,t}-\mathbf{z}_{v,t}^{(\mathrm{IDEAL})}\|_2$ & $\|\mathbf{s}_{v,t}-\mathbf{s}_{v,t}^{(\mathrm{IDEAL})}\|_2$ \\
        \midrule
        SGFormer (Single-view) & $11.2743\pm0.3400$ & $6.2460\pm0.9255$ \\
        \textbf{WorldGraph (Multi-view SGT)} & \textbf{3.8210} $\pm$ \textbf{0.5362} & \textbf{1.5568} $\pm$ \textbf{0.3691} \\
        \bottomrule
    \end{tabular}
\end{table}

\subsection{Validation of the RL Variance Bound}
\label{sec:ap-c-thm2-exp}
To empirically validate the ordering predicted by Theorem~\ref{thm:rl}, we conduct a controlled experiment on Node-level tasks/TGBN-Trade. Standard GRPO uses a fixed base group size $C_a$ for each transition type. Transition-aware RL retains the same base group size, then allocates the remaining rollout slots according to the estimated transition rarity until the configured total budget is exhausted. We report the first-order EVR approximation $1+(\delta-1)/(2\rho)$. As shown in Table~\ref{tab:rl_theorem_verification}, Transition-aware RL has a $13.02\%$ lower mean EVR than standard GRPO, consistent with the ordering established in Appendix~\ref{sec:ap-c-thm2}.

\begin{table}[!t]
    \centering
    \small
    \caption{Experimental validation of Theorem~\ref{thm:rl} on Node-level tasks/TGBN-Trade. Mean EVR is the arithmetic mean over the three transition groups; lower is better.}
    \label{tab:rl_theorem_verification}
    \setlength{\tabcolsep}{5pt}
    \renewcommand{\arraystretch}{1.08}
    \begin{tabular}{lc}
        \toprule
        Method & Mean EVR \\
        \midrule
        Standard GRPO & $1.5550$ \\
        \textbf{Transition-aware RL} & \textbf{1.3526} \\
        \bottomrule
    \end{tabular}
\end{table}
\newpage
\section{Time and Space Complexity Analysis of WorldGraph.}
\label{sec:B2}

Let $N=|V_t|$ and $E=|E_t|$ denote the numbers of nodes and edges in the current graph, respectively, and let $d$ be the maximum hidden dimension. We further use $L_{\max}$ for the number of message-passing hops, $M$ for the number of random walks per node, $T_{\mathrm{rw}}$ for their maximum length, $L_{\mathrm{hist}}$ for the history-window length, $C_t$ for the number of legal transition candidates at time $t$, and $\mathcal{B}$ for the total rollout budget. The time complexity of WorldGraph consists of the following three parts.

\subsection{Multi-Granularity Graph Encoding.}
The $L_{\max}$ hop-level message-passing layers require $O\!\left(L_{\max}(Ed+Nd^2)\right)$ time. Sampling $M$ walks of length at most $T_{\mathrm{rw}}$ from every node costs $O(NMT_{\mathrm{rw}})$, while path encoding and attention fusion require $O\!\left(N(MT_{\mathrm{rw}}+L_{\max}+M)d^2\right)$. Therefore, the total time complexity of this component is
\begin{equation}
O\!\left(L_{\max}Ed+N(L_{\max}+MT_{\mathrm{rw}}+M)d^2\right).
\end{equation}

\subsection{History-Aware State Encoding.}
For each node, the state encoder uses one current query to attend to at most $L_{\mathrm{hist}}$ historical memories. The query, key, value, and feed-forward projections require $O(NL_{\mathrm{hist}}d^2)$ time, while computing and aggregating the attention weights requires $O(NL_{\mathrm{hist}}d)$ time. Hence, this component has complexity $O(NL_{\mathrm{hist}}d^2)$.

\subsection{Transition-Aware RL.}
Encoding the node states and scoring the legal node or edge candidates costs $O\!\left((N+C_t)d^2\right)$. Dynamic group statistics and rollout allocation depend only on the small set of transition groups and are negligible compared with candidate scoring. The encoded graph state and Controller logits are shared across the $\mathcal{B}$ rollouts; batched transition-conditioned state updates and structure-aware reward computation require at most $O\!\left(\mathcal{B}(Nd^2+C_t+N+E)\right)$ time. For node- and graph-level tasks, $C_t$ is typically linear in $N$. For edge-level tasks, $C_t$ may reach $O(N^2)$ when all legal node pairs are considered, in which case candidate construction and scoring dominate the computation.

Combining the three components, the per-transition training complexity is
\begin{equation}
O\!\left(
L_{\max}Ed
+N(L_{\max}+MT_{\mathrm{rw}}+M+L_{\mathrm{hist}})d^2
+C_td^2
+\mathcal{B}(Nd^2+C_t+N+E)
\right).
\end{equation}
Since $L_{\max}$, $M$, $T_{\mathrm{rw}}$, $L_{\mathrm{hist}}$, and $\mathcal{B}$ are bounded hyperparameters, WorldGraph scales linearly with the observed graph size apart from the task-dependent transition-candidate set. Its space complexity is 
\begin{equation}
O\!\left(E+N(L_{\max}+MT_{\mathrm{rw}}+L_{\mathrm{hist}})d+C_td+\mathcal{B}(Nd+C_t)\right),   
\end{equation}
accounting for graph storage, hop/path representations, historical memories, candidate states, and rollout tensors.
}
\newpage

\section{Additional Datasets and Baselines Information}

\subsection{Statistics of Datasets}
\label{sec:C1}

We evaluate GWM-Zero on eight real-world dynamic graph datasets. Their statistics are summarized in Table~\ref{tab:dataset_stat}. Contact interactions are aggregated into causal six-hour snapshots, SocialEvo interactions into daily snapshots, and Enron interactions into weekly snapshots; the remaining datasets retain their released temporal granularity. We adopt the official TGB splits and chronological 70\%/15\%/15\% splits otherwise. All preprocessing statistics and thresholds are fitted using the training split only, and predictions use only the observed graph history.

\begin{table*}[!h]
    \centering
    \caption{Statistics of the eight real-world datasets used in GWM-Zero. The numbers of temporal events and snapshots are computed from the processed artifacts used by all methods.}
    \label{tab:dataset_stat}
    \scriptsize
    \resizebox{\textwidth}{!}{%
    \begin{tabular}{llrrrlll}
        \toprule
        Dataset & Domain & Nodes & Temporal Events & Snapshots & Edge Input & Snapshot Interval & Task(s) \\
        \midrule
        TGBN-Trade  & Economic interaction & 255    & 468{,}245    & 31    & Weighted & Annual          & Node-Level Tasks, Edge-Level Tasks \\
        TGBN-Genre  & User preference      & 1{,}505  & 17{,}858{,}395 & 1{,}580 & Weighted & Daily           & Node-Level Tasks \\
        TGBN-Reddit & Social interaction   & 11{,}766 & 27{,}174{,}118 & 1{,}090 & Binary   & Daily           & Node-Level Tasks \\
        UN Vote     & Political interaction& 201    & 1{,}035{,}742 & 72    & Weighted & Annual          & Edge-Level Tasks \\
        Contact     & Physical proximity   & 692    & 2{,}426{,}279 & 112   & Binary   & Six hours       & Edge-Level Tasks, Graph-Level Tasks \\
        SocialEvo   & Social proximity     & 74     & 2{,}098{,}119 & 243   & Binary   & Daily           & Edge-Level Tasks \\
        Flights     & Transportation       & 13{,}169 & 1{,}811{,}118 & 41    & Binary   & Monthly & Graph-Level Tasks \\
        Enron       & Communication        & 184    & 108{,}825    & 183   & Binary   & Weekly          & Graph-Level Tasks \\
        \bottomrule
    \end{tabular}}
\end{table*}

\begin{itemize}[left=-1pt]
    \item \textbf{TGBN-Trade}~\citep{huang2023temporal} is a directed international trade network. Nodes denote countries and weighted temporal edges record trade relations. Its official dynamic country-property vectors support the node-level task, while its annual topology is also used for edge-level prediction.

    \item \textbf{TGBN-Genre}~\citep{huang2023temporal} is a weighted bipartite interaction graph between users and music-genre coordinates. The official daily genre-preference vectors are retained without target remapping or normalization and are used for node-level evolution prediction.

    \item \textbf{TGBN-Reddit}~\citep{huang2023temporal} is a large bipartite interaction graph between entities and subreddit coordinates. We use binary graph connectivity and retain the official 698-dimensional dynamic property vectors for the node-level task.

    \item \textbf{UN Vote}~\citep{poursafaei2022towards} is an annual country interaction network. A directed weighted edge represents the number of co-yes voting relations between a country pair within an annual interval. It is used for edge addition and removal prediction.

    \item \textbf{Contact}~\citep{poursafaei2022towards} records physical proximity among students. We aggregate the released events into causal six-hour binary snapshots and use the resulting sequence for both edge-level and graph-level evaluation.

    \item \textbf{SocialEvo}~\citep{poursafaei2022towards} records temporal social-proximity interactions. After removing self-loops, we aggregate interactions into daily binary snapshots and use them for edge-level evaluation.

    \item \textbf{Flights}~\citep{poursafaei2022towards} is a directed airport network in which an edge denotes an observed flight connection. For graph-level evaluation, we use the released native temporal snapshots and predict the next snapshot's local structural changes.

    \item \textbf{Enron}~\citep{poursafaei2022towards} is a directed employee email network. We remove self-loops and aggregate the interactions into weekly binary snapshots for graph-level structural-evolution prediction.
\end{itemize}

\subsection{Statistics of Baselines}
\label{sec:C2}

We compare WorldGraph with eleven representative baselines covering graph encoders, graph Transformers, temporal graph models, pretrained graph models, and graph world models. Every method receives the same current graph, released history, data split, and task targets. For compatibility with the graph-world tasks, each baseline is equipped with supervised transition and count-prediction heads, while its original representation backbone is preserved.

\begin{itemize}[left=-1pt]
    \item \textbf{MPNN.}
    \begin{itemize}
        \item \textbf{GCN}~\citep{kipf2017semi} performs normalized neighborhood aggregation through graph convolution.
        \item \textbf{GAT}~\citep{velickovic2018graph} assigns learned attention weights to neighboring nodes during message passing.
        \item \textbf{GraphSAGE}~\citep{hamilton2017inductive} constructs node representations using sampled-neighborhood aggregation.
    \end{itemize}

    \item \textbf{Graph Transformers.}
    \begin{itemize}
        \item \textbf{SGFormer}~\citep{sgformer} combines simplified global self-attention with graph propagation for scalable graph representation learning.
        \item \textbf{NodeFormer}~\citep{wu2022nodeformer} approximates all-pair node attention with kernelized message passing.
        \item \textbf{GraphGPS}~\citep{RampasekGDLWB22} combines a local message-passing module with a global Transformer module.
    \end{itemize}

    \item \textbf{Temporal graph models.}
    \begin{itemize}
        \item \textbf{TGN}~\citep{rossi2020temporal} maintains node-wise memories that are updated by time-stamped interactions and uses them for temporal prediction.
        \item \textbf{TIDFormer}~\citep{peng2025tidformer} models temporal intervals and evolving dependencies through interval-aware attention.
    \end{itemize}

    \item \textbf{Pretrained graph models.}
    \begin{itemize}
        \item \textbf{Graph-JEPA}~\citep{skenderi2023graphjepa} improves the structured representation of a graph world model through train-only self-supervised latent prediction on graph patches, followed by fine-tuning of its encoder and task adapter.
        \item \textbf{MDGFM}~\citep{wang2025mdgfm} transfers a leave-one-domain-out pretrained encoder that aligns topology across source domains before downstream adaptation.
    \end{itemize}

    \item \textbf{Graph world model}\footnote{For L$^3$P and C-SWM, we preserve their original task-specific transition and decoder, and add WorldGraph's multi-granularity graph encoder, history-aware state encoder, Controller, transition-aware dynamic group sampling, and structure-aware rewards. GWM-E does not model transition-conditioned temporal state transitions, so we insert only our multi-granularity graph encoder as a fusion adapter over its multi-hop graph tokens and retain its original LLM decoder.}.
    \begin{itemize}
        \item \textbf{GWM-E}~\citep{FengWLY25}\footnote{\citet{FengWLY25} propose GWM-E and GWM-T. We compare our method with GWM-E, since it is publicly available and is reported as the stronger variant, whereas GWM-T is not publicly accessible.} adopts the publicly released embedding-based GWM (GWM-E). It repeatedly propagates node properties over the graph to obtain multi-hop embeddings, maps the embedding at each hop into graph tokens through hop-specific projectors, and inserts these tokens into the task prompt so that a frozen LLaMA decoder can produce the final prediction.
        \item \textbf{L$^3$P}~\citep{Zhang0S21a} identifies landmark states in the environment and constructs a reachability graph over them. It performs long-horizon planning by composing transitions between these landmarks, reducing the accumulation of errors from step-by-step prediction.
        \item \textbf{C-SWM}~\citep{KipfPW20} encodes visual observations into object-centric latent nodes, uses a graph neural network to model their interactions and transitions, and learns the latent dynamics with a contrastive objective.
    \end{itemize}
\end{itemize}

\subsection{Why We Choose These Datasets and Baselines}
\label{sec:C3}

The selected datasets cover economic, preference, social, political, proximity, transportation, and communication networks. They vary substantially in scale, temporal granularity, graph type, and edge semantics, ranging from 74 to 13{,}169 nodes and from approximately $10^5$ to $2.7\times10^7$ temporal events. Moreover, the same collection supports complementary node-, edge-, and graph-level predictions, allowing us to evaluate whether a graph world model can consistently capture evolution at different structural granularities rather than overfitting to one task type.

The baselines provide complementary comparisons. GCN, GAT, and GraphSAGE test conventional local message passing; SGFormer, NodeFormer, and GraphGPS test global graph attention; TGN and TIDFormer test dedicated temporal modeling; Graph-JEPA and MDGFM test whether graph pretraining alone is sufficient; L$^3$P and C-SWM test graph-structured planning and object-centric visual dynamics, respectively; and GWM-E tests graph-token conditioning of an LLM for traditional graph prediction tasks. This selection therefore separates the contributions of graph encoding, temporal state modeling, pretraining, graph-structured world modeling, graph-token task conditioning, and transition-aware reinforcement learning.
\newpage

\section{Experimental details}
\subsection{Details of WorldGraph Pretraining}
\label{sec:worldgraph_pretraining}
WorldGraph-Pretrain adopts same-task leave-one-dataset-out pretraining: for each target dataset, we exclude it and use chronological transitions from the remaining datasets under the same node-, edge-, or graph-level task. Given a current graph $g_t$, we construct two views through property masking and independent random-walk sampling. Representations of the same node $v$ in the two views form a positive pair, while representations of different sampled nodes serve as negatives; we optimize these pairs with a symmetric InfoNCE loss. We minimize the cosine distance between the two views' corresponding graph embeddings $\mathbf{z}_t$ and between their corresponding node-level latent states $\mathbf{s}_{v,t}$, and align the predicted next-step node representations with those encoded from $g_{t+1}$. The resulting multi-granularity graph encoder, history-aware state encoder, and latent predictor initialize the complete WorldGraph model for downstream training.

\subsection{Details of Task Information}
\label{expap:task}
\textbf{(1) Node Level Tasks.} We evaluate node addition, node deletion, and node property change on TGBN-Trade, TGBN-Genre, and TGBN-Reddit. 
{Each released snapshot defines one discrete time step. A node is naturally added when it has no incident edge at the current step but has one at the next, and naturally removed in the reverse case. Because natural changes are sparse, we supplement them with shared nodes whose historical interaction activity is rising (for addition) or declining (for deletion), masking the node, its incident edges, and its property at the corresponding time step.}
Properties are given by the official dynamic node-property labels. For the nodes observed at both adjacent steps, we compute their Jaccard distance between the Top-10 positive property dimensions. We then sort these distances in ascending order and use the value at the 60th percentile as the threshold. If a node’s Jaccard distance across adjacent steps exceeds this threshold, we mark it as indicating a property change.
We report F1 score for addition, deletion, and property change, and NDCG@10 for future property prediction on changed nodes.

\textbf{(2) Edge Level Tasks.} We evaluate edge addition and deletion on TGBN-Trade, UN Vote, Contact, and SocialEvo. Given the current graph, the model predicts which disconnected node pairs will form edges and which existing edges will disappear, and then reconstructs the next graph. We report F1 score for edge addition and deletion and their macro-F1 score.

\textbf{(3) Graph Level Tasks.} We evaluate {subgraph-level structural property changes} on Flights, Contact, and Enron datasets. 
{We use all nodes incident to at least one edge in the current graph as center nodes. For each selected node, we construct its local subgraph in both the current and next graphs, and predict the change in its six-dimensional structural property vector, consisting of node count, edge count, density, number of connected components, clustering coefficient, and triangle count.}
We report MAE and RMSE for these changes and macro-F1 for classifying them as decreasing, unchanged, or increasing.

\subsection{Details of Plug-in Integration and Traditional Graph Prediction}
\label{sec:ap_details_plugin_and_traditional}
As discussed in Section~\ref{sec:Comparison on Existing GWM Problem Settings}, we evaluate WorldGraph in the following two settings. For L$^3$P and C-SWM, we preserve each model's original backbone and task-specific decoder while adding the multi-granularity graph encoder, history-aware state encoder, Controller, transition-aware dynamic group sampling, and structure-aware reward. Separately, we compare standalone WorldGraph with GWM-E on the latter's traditional graph prediction tasks, where GWM-E values are taken from the original reported results.

\subsection{Details of Hardware Information}
The main experiments are conducted on a server equipped with an Intel Xeon E5-2698 v4 CPU at 2.20 GHz (40 logical CPU cores), 256 GB RAM, and four NVIDIA Tesla V100 GPUs with 16 GB memory each. Each training process uses one GPU.

\subsection{Details of Parameter setup.}
All methods use the same chronological splits, graph inputs, history, and task targets. WorldGraph is optimized with AdamW using an eight-state history window. Checkpoints are selected on the validation set, and results are averaged over five runs. The main settings are summarized in Table~\ref{tab:hyperparams}.

\begin{table}[!t]
    \centering
    \caption{Principal hyperparameters used by WorldGraph.}
    \label{tab:hyperparams}
    \small
    \begin{tabular}{lccc}
        \toprule
        Hyperparameter & Node-level tasks & Edge-level tasks & Graph-level tasks \\
        \midrule
        Latent-state dimension & 64 & 32 & 64 \\
        Hidden-state dimension & 64 & 32 & 64 \\
        Transition-representation dimension & 32 & 16 & 32 \\
        Maximum historical time span & 8 & 8 & 8 \\
        Dropout rate & 0.1 & 0.1 & 0.1 \\
        Learning rate & $10^{-3}$ & $10^{-3}$ & $10^{-3}$ \\
        Total rollout budget & 12 & 16 & 12 \\
        \bottomrule
    \end{tabular}
\end{table}

All graph-state transitions are processed chronologically without access to future information, and every method uses the same task construction and evaluation protocol.
\newpage

{\section{More Experiments}}
\subsection{Efficiency Analysis of WorldGraph.}
We compare the mean epoch time, computed by averaging the first five training epochs, for WorldGraph and the eight baseline models on TGBN-Genre. The last two rows in Table~\ref{tab:genre_efficiency} report only the pure reinforcement-learning overhead for standard GRPO and Transition-aware RL, respectively. WorldGraph w/o RL is somewhat slower than most representation and temporal baselines, but it is not the slowest; given the substantially stronger predictive performance reported in Table~\ref{tab:worldgraph_results}, this moderate computational cost is acceptable. For the RL comparison, Transition-aware RL incurs only $1.99$~s more per epoch than standard GRPO ($27.54$~s compared with $25.55$~s), indicating that the two key components of our RL design---transition-aware dynamic group sampling and structure-aware reward---introduce little additional runtime. Overall, these results show that WorldGraph provides a reasonable overall efficiency while delivering stronger predictive performance.

\begin{table}[!h]
    \centering
    \caption{Mean training time per epoch (seconds) on TGBN-Genre, computed over the first five training epochs.}
    \label{tab:genre_efficiency}
    \small
    \setlength{\tabcolsep}{8pt}
    \renewcommand{\arraystretch}{1.08}
    \begin{tabular}{lr}
        \toprule
        Model & Time (s/epoch) \\
        \midrule
        GCN & 47.26 \\
        GAT & 51.47 \\
        GraphSAGE & 43.49 \\
        SGFormer & 53.45 \\
        NodeFormer & 72.11 \\
        GraphGPS & 64.04 \\
        TGN & 44.50 \\
        TIDFormer & 55.10 \\
        \textbf{WorldGraph w/o RL} & 66.89 \\
        \midrule
        Standard GRPO & 25.55 \\
        \textbf{Transition-aware RL} & 27.54 \\
        \bottomrule
    \end{tabular}
\end{table}

\subsection{Multi-step Prediction Analysis.}
\label{sec:multistep_graph_prediction}

We evaluate multi-step prediction on the \emph{Graph-level tasks}, where the model forecasts changes in the structural property vector of local subgraphs. Each model is first trained with the standard one-step Graph-level protocol. At evaluation time, it observes the graph at an initial time, predicts the next graph-state change, and then recursively feeds its own predicted graph state into the same state-transition and prediction modules for steps $2$, $3$, and $5$; ground-truth future graphs are used only to construct the evaluation targets. Following the Graph-level evaluation protocol, we report MAE and RMSE for these structural-property changes and Macro-F1 for classifying them as decreasing, unchanged, or increasing. Table~\ref{tab:multistep_graph} summarizes the results. WorldGraph-Pretrain is the best method on all nine Contact and Enron metrics, and on eight of the nine Flights metrics. Its Macro-F1 improvement over the strongest external baseline is $1.78\%/3.90\%/4.12\%$ on Flights, $6.36\%/6.19\%/7.60\%$ on Contact, and $1.10\%/1.45\%/1.28\%$ on Enron for Steps $2/3/5$, respectively. At Step 5, it also reduces MAE/RMSE by $0.4075/0.5832$ on Flights, $0.3861/0.5506$ on Contact, and $0.3789/0.4997$ on Enron relative to the best external values. These results show that our models retain a clear advantage as the prediction step increases, demonstrating effective multi-step graph-state prediction.

\begin{table*}[!t]
    \centering
    \scriptsize
    \caption{Multi-step prediction results on the Graph-level tasks. Models recursively use their predicted graph states for steps $2$, $3$, and $5$. MAE and RMSE are lower-is-better, whereas Macro-F1 is higher-is-better. The best results are shown in \textbf{bold} and the second best results are \underline{underlined}.}
    \label{tab:multistep_graph}
    \resizebox{\textwidth}{!}{%
    \begin{tabular}{llrrrrrrrrr}
        \toprule
        \multirow{2}{*}{Dataset} & \multirow{2}{*}{Model} & \multicolumn{3}{c}{Step 2} & \multicolumn{3}{c}{Step 3} & \multicolumn{3}{c}{Step 5} \\
        \cmidrule(lr){3-5} \cmidrule(lr){6-8} \cmidrule(lr){9-11}
        & & MAE & RMSE & Macro-F1 & MAE & RMSE & Macro-F1 & MAE & RMSE & Macro-F1 \\
        \midrule
        \multirow{12}{*}{\textit{Flights}} & GCN & 0.9208 & 1.5297 & 0.4486 & 1.3097 & 2.0548 & 0.3970 & 2.2527 & 3.2212 & 0.3584 \\
        & GAT & 0.9080 & 1.4992 & 0.4314 & 1.1381 & 1.7816 & 0.3950 & 1.7601 & 2.5047 & 0.3372 \\
        & GraphSAGE & 0.8764 & 1.4712 & 0.4667 & 1.1325 & 1.8040 & 0.4418 & 1.6862 & 2.4935 & 0.3982 \\
        & SGFormer & 0.8789 & 1.4686 & 0.4735 & 1.1885 & 1.8848 & 0.4207 & 1.8938 & 2.7597 & 0.3861 \\
        & NodeFormer & 0.9694 & 1.5545 & 0.4004 & 1.3759 & 1.9958 & 0.3755 & 2.2803 & 2.9962 & 0.3449 \\
        & GraphGPS & 0.8799 & 1.4658 & 0.4834 & 1.2033 & 1.8941 & 0.4433 & 1.9271 & 2.8334 & 0.4045 \\
        & TGN & 0.9161 & 1.4559 & 0.4467 & 1.4293 & 2.0275 & 0.3755 & 3.0440 & 3.8768 & 0.3411 \\
        & TIDFormer & 0.8942 & 1.4159 & 0.4682 & 1.2958 & 1.9056 & 0.4236 & 2.5503 & 3.5035 & 0.3537 \\
        & Graph-JEPA & 0.9502 & 1.5329 & 0.4564 & 1.6205 & 2.3081 & 0.4047 & 4.1542 & 5.1787 & 0.3511 \\
        & MDGFM & 0.9480 & 1.5371 & 0.3762 & 1.3147 & 1.9468 & 0.3512 & 2.1888 & 2.9197 & 0.3284 \\
        \cmidrule(lr){2-11}
        & \textbf{WorldGraph} & \textbf{0.7248} & \underline{1.2904} & \underline{0.4889} & \underline{0.8976} & \underline{1.4458} & \underline{0.4694} & \underline{1.3700} & \underline{2.0019} & \underline{0.4376} \\
        & \textbf{WorldGraph-Pretrain} & \underline{0.7251} & \textbf{1.2879} & \textbf{0.5012} & \textbf{0.8660} & \textbf{1.4123} & \textbf{0.4823} & \textbf{1.2787} & \textbf{1.9103} & \textbf{0.4457} \\
        \midrule
        \multirow{12}{*}{\textit{Contact}} & GCN & 0.8891 & 1.1747 & 0.4889 & 1.3626 & 1.7485 & 0.4343 & 2.3580 & 2.8872 & 0.3735 \\
        & GAT & 0.8579 & 1.1269 & 0.4862 & 1.2642 & 1.6818 & 0.4285 & 1.9412 & 2.7057 & 0.4048 \\
        & GraphSAGE & 0.8653 & 1.1392 & 0.5057 & 1.2461 & 1.6131 & 0.4657 & 1.6471 & 2.1265 & 0.4286 \\
        & SGFormer & 0.8160 & 1.1041 & 0.4935 & 1.1923 & 1.5899 & 0.4422 & 1.6425 & 2.2004 & 0.4057 \\
        & NodeFormer & 1.0261 & 1.3081 & 0.4596 & 1.7076 & 2.1143 & 0.4068 & 2.8412 & 3.3949 & 0.3751 \\
        & GraphGPS & 0.8513 & 1.1307 & 0.4971 & 1.3215 & 1.6957 & 0.4505 & 2.0084 & 2.5280 & 0.4244 \\
        & TGN & 1.0661 & 1.3535 & 0.4750 & 1.8515 & 2.3242 & 0.4089 & 3.5733 & 4.3327 & 0.3650 \\
        & TIDFormer & 0.9457 & 1.2559 & 0.4828 & 1.6557 & 2.1509 & 0.4324 & 3.3123 & 4.2178 & 0.3936 \\
        & Graph-JEPA & 0.9747 & 1.2581 & 0.4884 & 1.6361 & 2.1350 & 0.4388 & 3.0525 & 4.0853 & 0.4112 \\
        & MDGFM & 0.9567 & 1.2312 & 0.4642 & 1.5444 & 1.9985 & 0.4085 & 2.8350 & 3.6336 & 0.3985 \\
        \cmidrule(lr){2-11}
        & \textbf{WorldGraph} & \underline{0.6756} & \underline{0.9754} & \underline{0.5412} & \underline{0.8290} & \underline{1.1253} & \underline{0.4984} & \underline{1.3388} & \underline{1.6578} & \underline{0.4919} \\
        & \textbf{WorldGraph-Pretrain} & \textbf{0.6540} & \textbf{0.9542} & \textbf{0.5693} & \textbf{0.7796} & \textbf{1.0800} & \textbf{0.5276} & \textbf{1.2564} & \textbf{1.5759} & \textbf{0.5046} \\
        \midrule
        \multirow{12}{*}{\textit{Enron}} & GCN & 1.3042 & 1.9075 & 0.3605 & 1.5877 & 2.1316 & 0.3126 & 2.2529 & 2.8909 & 0.2909 \\
        & GAT & 1.2605 & 1.9077 & 0.3580 & 1.3940 & 1.9723 & 0.3290 & 1.6845 & 2.2817 & 0.2990 \\
        & GraphSAGE & 1.2994 & 1.9115 & 0.3718 & 1.5150 & 2.0678 & 0.3369 & 1.9448 & 2.5415 & 0.3059 \\
        & SGFormer & 1.2876 & 1.8868 & 0.3794 & 1.5019 & 2.0407 & 0.3455 & 1.8557 & 2.4450 & 0.3211 \\
        & NodeFormer & 1.2769 & 1.9222 & 0.3948 & 1.4932 & 2.0900 & 0.3688 & 1.9648 & 2.6402 & 0.3424 \\
        & GraphGPS & 1.2499 & 1.8518 & 0.3937 & 1.3978 & 1.9342 & 0.3546 & 1.7089 & 2.2872 & 0.3152 \\
        & TGN & 1.2283 & 1.8633 & 0.4151 & 1.3613 & 1.9201 & 0.3968 & 1.6981 & 2.2845 & 0.3809 \\
        & TIDFormer & 1.2737 & 1.8897 & 0.3761 & 1.4000 & 1.9510 & 0.3558 & 1.7766 & 2.3845 & 0.3299 \\
        & Graph-JEPA & 1.3478 & 1.9117 & 0.3394 & 1.6037 & 2.1119 & 0.3145 & 2.2833 & 2.9209 & 0.2981 \\
        & MDGFM & 1.3762 & 2.0157 & 0.3419 & 1.6272 & 2.1908 & 0.3015 & 2.2030 & 2.8013 & 0.2610 \\
        \cmidrule(lr){2-11}
        & \textbf{WorldGraph} & \underline{1.2231} & \underline{1.7935} & \underline{0.4172} & \underline{1.2508} & \underline{1.7171} & \underline{0.4108} & \underline{1.3441} & \underline{1.7830} & \underline{0.3901} \\
        & \textbf{WorldGraph-Pretrain} & \textbf{1.2071} & \textbf{1.7874} & \textbf{0.4261} & \textbf{1.2332} & \textbf{1.7136} & \textbf{0.4113} & \textbf{1.3192} & \textbf{1.7820} & \textbf{0.3937} \\
        \bottomrule
    \end{tabular}}
\end{table*}
\begin{figure*}[!t]
    \centering
    \includegraphics[width=0.6\linewidth]{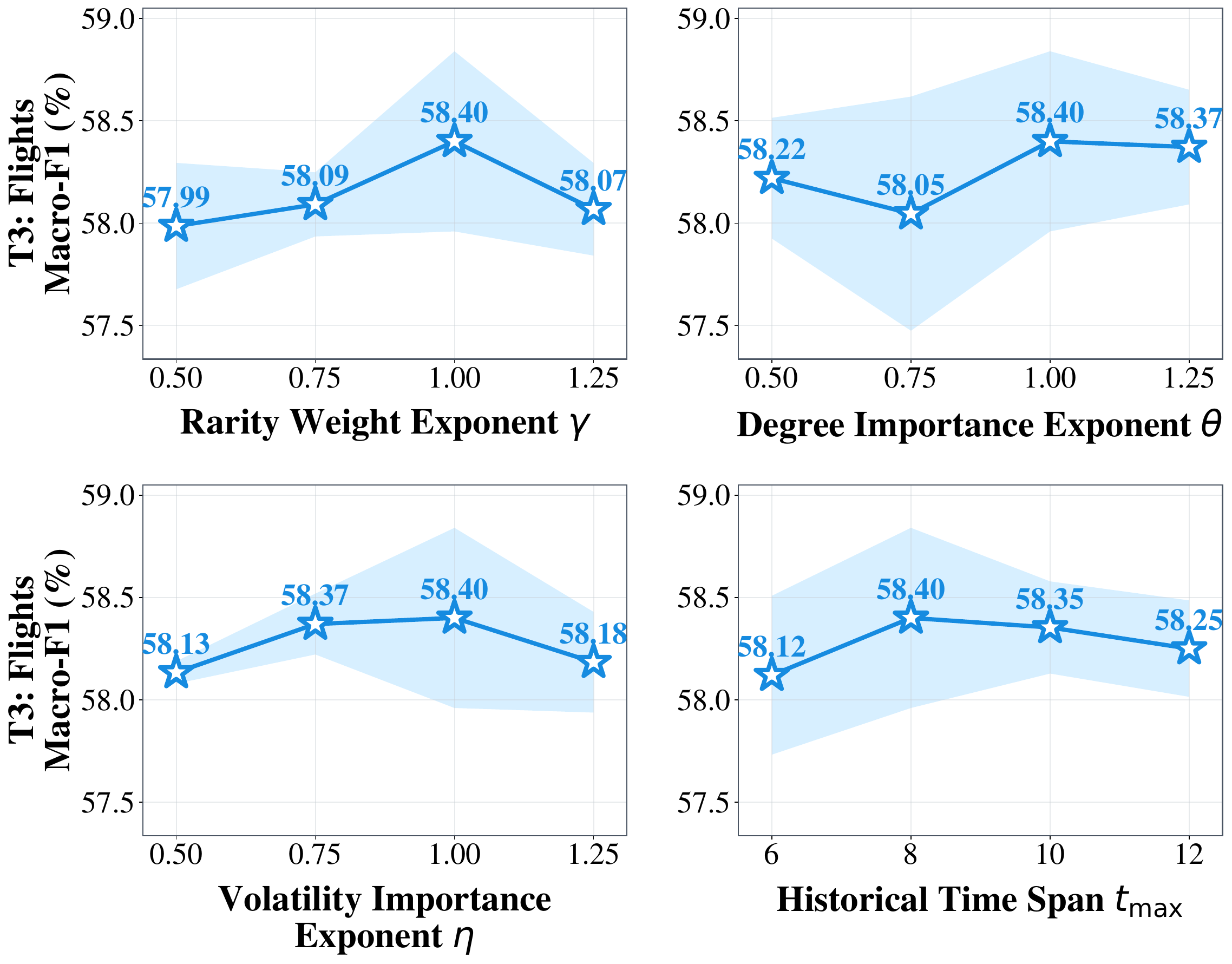}
    \caption{Sensitivity analysis on Graph-Level Tasks/Flights (Macro-F1) across the rarity weight exponent $\gamma$, degree importance exponent $\theta$, volatility importance exponent $\eta$, and historical time span $t_\mathrm{max}$.}
    \label{fig:rl_sensitivity}
\end{figure*}
\subsection{Additional Sensitivity Analysis}
\label{sec:ap_rl_sensitivity}
Figure~\ref{fig:rl_sensitivity} shows the sensitivity of WorldGraph to four additional hyperparameters on Graph-Level Tasks/Flights (Macro-F1): the rarity weight exponent $\gamma$, the degree importance exponent $\theta$, the volatility importance exponent $\eta$, and the historical time span $t_\mathrm{max}$ used by the history-aware state encoder. Overall, the four curves exhibit an inverted U-shaped trend, showing that moving either below or above the default configuration can degrade performance. The default values $\gamma=1.00$, $\theta=1.00$, $\eta=1.00$, and $t_\mathrm{max}=8$ reach the common peak of $58.40\%$; for example, the rarity exponent gives $57.99\%$ at $0.50$ and $58.07\%$ at $1.25$, while the historical span gives $58.13\%$, $58.35\%$, and $58.25\%$ at $t_\mathrm{max}=6$, $10$, and $12$, respectively. The intermediate settings of $\theta$ and $\eta$, together with the default history span, provide a balance between emphasizing important nodes, preserving structural coverage, and retaining useful temporal context.

\subsection{Image-to-Graph Visual Rollout.}
We further evaluate WorldGraph on a five-object Shapes transition-conditioned visual rollout task. Each input image is encoded by a CNN object extractor and an MLP into five object nodes forming a complete graph. Given the recorded environment transition, the multi-granularity graph encoder, history-aware state encoder, Controller, and transition module predict the next graph-state change; the Controller is trained with transition-aware dynamic group sampling and structure-aware rewards. A separately trained CNN decoder then maps the updated graph state back to an image. During multi-step prediction, only the initial image and the recorded transition sequence are provided, and the model recursively uses its predicted graph state rather than any ground-truth future image. As shown in Figure~\ref{fig:vision_worldgraph_rollout}, the decoded WorldGraph rollout closely matches the ground-truth sequence, with the object positions and their temporal changes remaining well aligned across the prediction steps. This result demonstrates that WorldGraph can also support computer-vision tasks through graph-state prediction and reconstruction.

\begin{figure*}[!h]
    \centering
    \includegraphics[width=\linewidth]{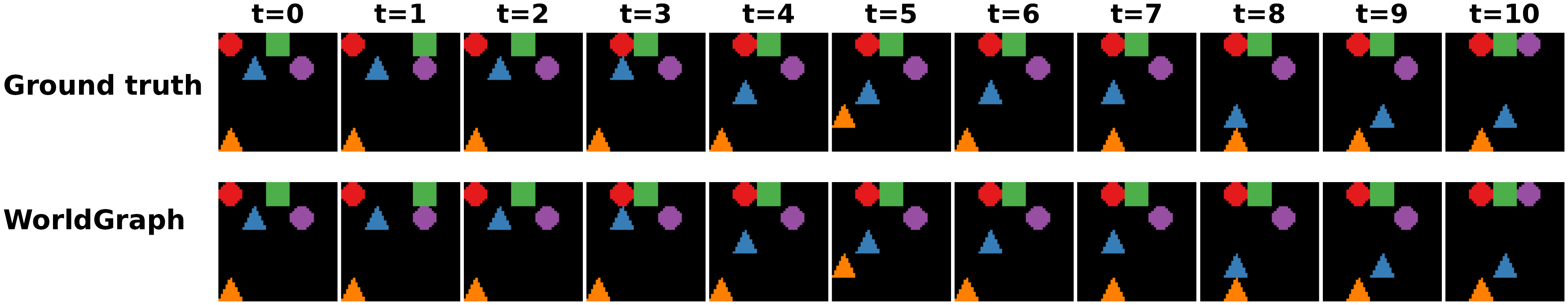}
    \caption{Transition-conditioned image-to-graph-to-image rollout. The first row shows the ground-truth sequence, and the second row shows the images decoded from the graph states predicted by WorldGraph.}
    \label{fig:vision_worldgraph_rollout}
\end{figure*}
\clearpage
\newpage
\section{More Related Works}
\subsection{World Models.}
World models learn compact internal representations of environments and model their temporal evolution to support prediction, planning, and decision-making. The seminal World Models framework~\citep{worldmodels} combines visual representation learning, recurrent memory, and a controller, while PlaNet~\citep{hafner2019learning} learns stochastic latent dynamics for online planning and Dreamer~\citep{HafnerLB020} learns long-horizon behaviors through latent imagination. Recent studies further improve the scalability and temporal abstraction of world models through decoder-free latent dynamics~\citep{hansen2024td}, Transformer-based contrastive prediction~\citep{burchi2025learning}, and autoregressive next-clip diffusion for visual prediction~\citep{zhuang2026video}. Object-centric world models further decompose observations into explicit entities and model their stochastic dynamics~\citep{daniel2026latent}. However, although these representations improve prediction or make individual entities explicit, they generally leave evolving relations and structural dependencies implicit. Consequently, they struggle to model entity, relation, and attribute changes at different granularities, highlighting graph-structured representations and relational inductive biases as important foundations for modeling complex and evolving worlds~\citep{liu2026graphworldmodelsconcepts}.

\subsection{Temporal Graph Learning.}
Temporal graph learning models~\citep{kumar2019predicting,rossi2020temporal,xu2020inductive,wang2021inductive,cong2023we,yu2023towards,tian2024freedyg,lu2024improving,zou2024repeat,peng2025tidformer,shi2026temporal} time-varying interactions and connectivity, typically by encoding timestamped events or graph snapshots into temporal node representations. 
For example, TGN~\citep{rossi2020temporal} maintains node memories from timed events, TIDFormer~\citep{peng2025tidformer} uses calendar-based temporal partitioning and interaction-level attention to capture temporal and interactive dynamics, and TGT~\citep{shi2026temporal} summarizes global evolutionary regularities with a temporal graph thumbnail. However, these methods primarily organize temporal information around interactions, sampled neighborhoods, or aggregate graph regularities and are mainly evaluated on node- or link-oriented prediction tasks. They are therefore not designed to explicitly represent coordinated changes to entities, relations, attributes, and local structures at multiple granularities. In contrast, WorldGraph models such structured graph changes as explicit transitions and jointly performs state modeling and node-, edge-, and graph-level transition prediction.

\section{Reproducibility and Code Availability}
To ensure the reproducibility of our results, we provide the source code of WorldGraph and the benchmark datasets (GWM-Zero) at \url{https://github.com/USTC-DataDarknessLab/Graph-Native_World_Modeling}.

\section{Limitations and  Broader Impacts}
\subsection{Limitations}
\label{sec:E1}
In this work, we present an initial conceptualization of a graph-native world model. However, further work is still needed to better integrate information from diverse modalities, so as to more fully realize the potential of graph world models.
\subsection{Broader Impacts}
\label{sec:E2}
This work proposes a novel graph-native world modeling paradigm, with the potential to address existing limitations of conventional world models in learning structural information, thereby enabling more faithful modeling of relational dynamics in evolving graph worlds.
Meanwhile, the core ideas of this work can be transferred to different modalities, enabling multimodal graph world modeling where diverse observations (e.g., images and text) are fused into an evolving graph representation.

\newpage

\end{document}